\documentclass{article}
\usepackage[latin9]{inputenc}
\usepackage{geometry}
\usepackage{float}
\usepackage{calc}
\usepackage{mathtools}
\usepackage{url}
\usepackage{amsmath}
\usepackage{amssymb}
\usepackage{graphicx}
\usepackage[authoryear]{natbib}
\usepackage[]{hyperref}
\usepackage{breakurl}

\makeatletter

\providecommand{\tabularnewline}{\\}
\floatstyle{ruled}
\newfloat{algorithm}{tbp}{loa}
\providecommand{\algorithmname}{Algorithm}
\floatname{algorithm}{\protect\algorithmname}

\@ifundefined{date}{}{\date{}}

\usepackage{times}

\usepackage[table]{xcolor}

\usepackage{amsthm}

\theoremstyle{definition}
\newtheorem{definition}{Definition}[section]\newtheorem{proposition}{Proposition}[section]
\newtheorem{example}{Example}[section]
\newtheorem{remark}{Remark}[section]
\newtheorem{lemma}{Lemma}[section]

\numberwithin{equation}{section}
\numberwithin{figure}{section}
\numberwithin{table}{section}

\newcommand{\edg}{\mathcal{E}}

\newcommand{\vpa}{vpa}

\newcommand{\spgm}{\mathbb{S}}

\newcommand{\enctr}{\mathcal{T}}
\newcommand{\val}{\Delta}

\usepackage{tikz}
\def\checkmark{\tikz\fill[scale=0.4](0,.35) -- (.25,0) -- (1,.7) -- (.25,.15) -- cycle;}

\newcommand{\tableyes}{\checkmark}

\usepackage{algorithm}
\usepackage[noend]{algpseudocode}
\algnewcommand\algorithmicinput{\textbf{Input:}}
\algnewcommand \Input{\item[\algorithmicinput]}
\algnewcommand\algorithmicoutput{\textbf{Output:}}
\algnewcommand \Output {\item[\algorithmicoutput]}
\algnewcommand\And{\textbf{and }}
\algnewcommand\Or{\textbf{or }}

\title{
{ \textsc {\LARGE Sum-Product Graphical Models}\footnote{Support of the German Science Foundation, grant GRK 1653, is gratefully acknowledged.}}
}

\author{
Mattia Desana${}^1$\thanks{Corresponding author. Email: mattia.desana@iwr.uni-heidelberg.de. }
$\,$ and Christoph Schn\"orr${}^{1,2}$.
}

\date{${}^1$Heidelberg Collaboratory for Image Processing (HCI)\\
${}^2$Image and Pattern Analysis Group (IPA)\\
Heidelberg University, Germany}

\@ifundefined{showcaptionsetup}{}{%
 \PassOptionsToPackage{caption=false}{subfig}}
\usepackage{subfig}
\makeatother

\begin{document}
\maketitle
\begin{abstract}
This paper introduces a new probabilistic architecture called Sum-Product
Graphical Model (SPGM). SPGMs combine traits from Sum-Product Networks (SPNs) 
and Graphical Models (GMs): Like SPNs, SPGMs always enable tractable inference using a class of models that incorporate context specific independence. 
Like GMs, SPGMs provide a high-level model interpretation in terms of conditional independence assumptions and corresponding factorizations. Thus, the new architecture represents a class of probability distributions that combines, for the first time, the semantics of graphical models with the evaluation efficiency
of SPNs. 
We also propose a novel algorithm for learning both the structure and the parameters of SPGMs. A comparative empirical evaluation demonstrates competitive performances of our approach in density estimation. 

\end{abstract}


\section{Introduction}
The compromise between model expressiveness and tractability of model evaluation (inference) is a key issue of scientific computing. 
Regarding probabilistic \textit{Graphical Models (GMs)}, tractable inference is guaranteed for acyclic graphical models and GMs on cyclic graphs with small treewidth \citep{Wainwright:2008}, i.e.~on graphs that after triangultradeation admit a tree-decomposition which induces only maximal cliques of small size \citep{Diestel-06}. On the other hand, except for a subset of discrete graphical models (see, e.g., \cite{EnergyGraphCuts-PAMI04}) where inference can be reformulated as a maximum flow problem, inference with cyclic graphical models generally suffers from a complexity that exponentially grows with the treewidth of the underlying graph, so that approximate inference is the only viable choice. 

trade
Recently, \emph{Sum-Product Networks (SPNs)} 
\citep{SPN2011} and closely related architectures including Arithmetic Circuits and And-Or Graphs \citep{ArithmCircuitsAndNetworkPoly2,AndOrDechter200773} have received attention in the probabilistic machine learning community, mainly due to two attractive properties:
\begin{enumerate}
\item
These architectures allow to cope with probability distributions that are more complex than tractable graphical models as characterized above. A major reason is that SPNs enable an efficient representation of \emph{contextual independency}: independency between variables that only holds in connection with some assignment of a subset of variables in the model, called ``context''. Exploiting contextual independency allows to drastically reduce the cost of inference, whenever the modelled distribution justfies this assumption. By contrast, as discussed by \citet{Boutilier96context-specificindependence}, GMs cannot represent contextual independence compactly, since the connection between nodes in a GM represent \textit{conditional} independences rather than \textit{contextual} ones. As a result, a significant subset of distributions that would be represented by graphical models with high treewidth (due to the inability to exploit contextual independency) can be represented by SPNs in a tractable way. A detailed example illustrating this key point is provided in Section \ref{subsec:An-Illustrative-Example}.
\item
Secondly, SPNs ensure that the cost of inference is always \emph{linear} in the model size and, therefore, inference is always tractable. This aspect greatly simplifies approaches to learning the structure of such models, the complexity of which essentially depends on the complexity of inference as a subroutine.
In recent work, it has been shown empirically that 
structure learning approaches for  SPNs produce state of the art results
in density estimation (see e.g. \citet{SPNstructureLearning2013},
\citet{Rooshenas14}, \citet{Rahman16}, \citet{Rahman16mergeSPN}), suggesting that performing exact inference with tractable
models might be a better approach than approximate
inference using more complex but intractable models. 
\end{enumerate}

On the other hand, the ability of SPNs to represent efficiently contextual independency is due to a \textit{low-level} representation of the underlying distribution. This representation comprises a Directed Acyclic Graph with sums and products as internal nodes, and with indicator variables associated to each state of each variable in the model, that become active
when a variable is in a certain state (Fig.~\ref{fig:Representation-properties-of}(b)). 

Thus, SPN graphs directly represent the 
flow of operations performed during the inference procedure, which is much harder to read and interpret than a factorized graphical model due to conditional independence. 
In particular, the factorization associated to the graphical model is lost after translating the model into a SPN, and can only be retrieved (when possible) through a complex hidden variable restoration procedure \citep{peharz15}. As a consequence of these incompatibilities, research on SPNs has largely evolved without much overlap with work on GMs.

The focus of this paper is to exploit jointly the favourable properties of GMs and SPNs.
This has also been the objective of several related papers, which aimed at either endowing GMs with contextual independency
or at extracting probabilistic semantics from SPNs. These prior works are discussed in detail in Section \ref{sec:Related-Architectures} along with Table  \ref{table:properties}. Rather than extending an existing model, however, we 
introduce a new archiifundefinedtecture which directly inherits the complementary favourable properties of both representations: 
conditional and contextual independency, together with tractable inference. We call this new architecture \textit{Sum-Product Graphical Model (SPGM)}. 


In addition, we devise a novel algorithm for learning both the structure and the model parameters of SPGMs which exploits the \textit{connection} between GMs and SPNs. A comparative empirical evaluation demonstrates competitive performance of our approach in density estimation: we obtain results close to state of the art models, despite using a radically different approach from the established body of work and comparing against methods that rely on years of intensive 
research. These results demonstrate that SPGMs cover an attractive subclass of probabilistic models that can be efficiently evaluated and that are amenable to model parameter learning.



\subsection{Tradeoff Between High-Level Representation and Efficient Inference: An Example} \label{subsec:An-Illustrative-Example}

\begin{figure}[tb]
\quad{}\quad{}\quad{}\quad{}\quad{}\quad{}\quad{}\quad{}\quad{}\quad{}\subfloat[\label{fig:gm1}]{\includegraphics[width=0.1\textwidth]{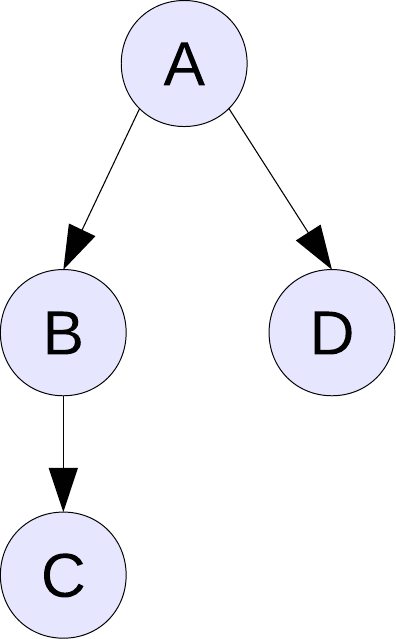} 

}\hfill{}\subfloat[\label{fig:spn1}]{\includegraphics[width=0.25\textwidth]{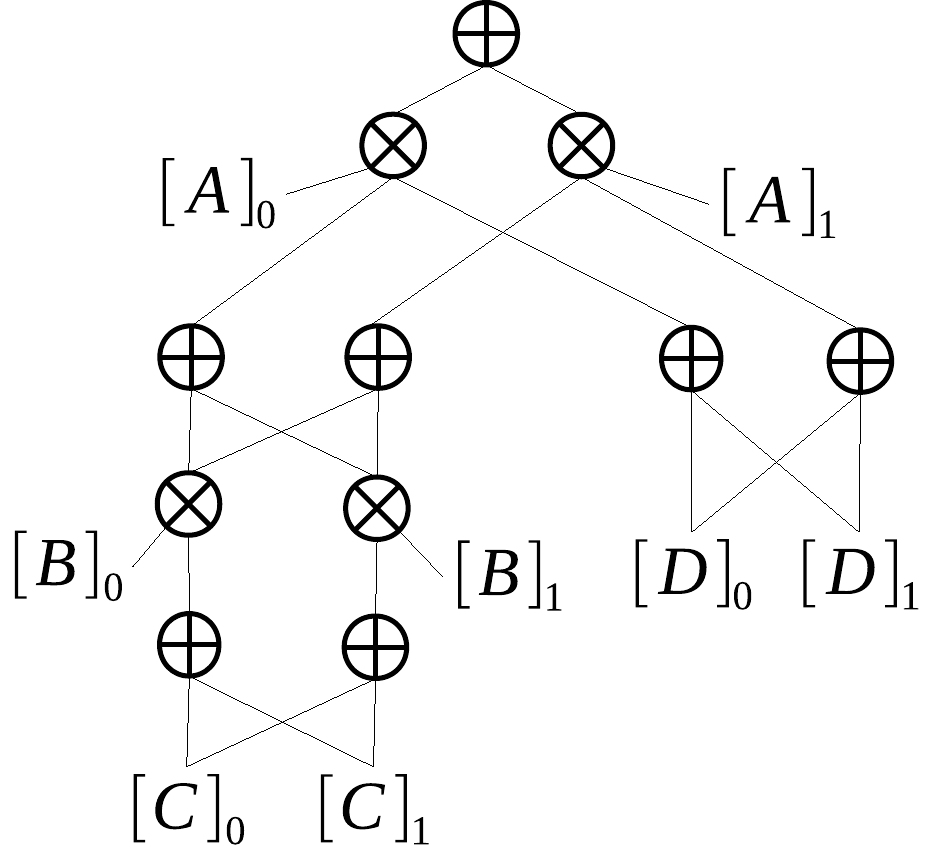} 

}\quad{}\quad{}\quad{}\quad{}\quad{}\quad{}\quad{}\quad{}\quad{}\quad{}\caption{Representation properties of Graphical Models (GMs) and Sum-Product
Networks (SPNs). The same distribution specified by  Eq.~\eqref{eq:ex1} is 
represented by a GM in panel (a) and by a SPN in panel (b). This illustrates that GMs represent conditional independence
more compactly than SPNs. \label{fig:Representation-properties-of}}
\end{figure}


We consider a distribution of discrete random variables $A,B,C,D$ in the following form (shown in Fig.~\ref{fig:gm1} as a directed Graphical Model (GM)):
\begin{equation}
P(A,B,C,D)=P(A)P(B|A)P(C|B)P(D|A)\label{eq:ex1}
\end{equation}
Uppercase letters $A, B, C, D$ denote random variables and corresponding lowercase
letters $a, b, c, d$ values in their domains $\val(A), \val(B), \val(C), \val(D)$.
We write
$\sum_{a,b,c,d}$ for the sum over the joint domain $\val(A) \times \val(B) \times \val(C) \times \val(D)$.
Using this notation, the distribution $P(A,B,C,D)$ can be written
in the form of a \textit{network polynomial} (\cite{Darwiche2003}) as:
\begin{equation}
P(A,B,C,D)=\sum_{a,b,c,d}P(a,b,c,d)[A]_{a}[B]_{b}[C]_{c}[D]_{d}\label{eq:ex2}
\end{equation}
Here $P(a,b,c,d)$ denotes the \emph{value} of $P$ for assignment
$A=a,B=b,C=c,D=d$, and $[A]_{a},[B]_{b},[C]_{c},[D]_{d} \in \{0,1\}$ denote \emph{indicator variables}.
For instance, to compute
the partition function all indicator variables of \eqref{eq:ex2} are set to $1$, and to compute
the marginal probability $P(A=1)$ one sets $[A]_{1}=1,[A]_{0}=0$
and all the remaining indicators to $1$. 

The next step is to exploit the factorization of $P$ on the right-hand side of \eqref{eq:ex1} in order to rearrange the sum of \eqref{eq:ex2} more economically in terms of messages 
$\mu$, which results in the sum-product message passing formulation equivalent to \eqref{eq:ex2},
\begin{subequations}\label{eq:subeq1}
\begin{align}
P(A,B,C,D) & =\sum_{a\in\val(A)}P(a)[A]_{a}\mu_{b,a}(a)\mu_{d,a}(a), &
\mu_{b,a}(A) & =\sum_{b\in\val(B)}P(b|A)[B]_{b} \mu_{c,b}(b), \\
\mu_{c,b}(B) &=\sum_{c\in\val(C)}P(c|B)[C]_{c}, &
\mu_{d,a}(A) &=\sum_{d\in\val(D)}P(d|A)[D]_{d}.
\end{align}
\end{subequations}
As discussed above, the distribution can be represented in two forms: The first one is a directed graphical model conforming to Eq.~\eqref{eq:ex1}, shown in Fig.~\ref{fig:gm1}. The second one is a sum-product network (SPN) shown by Fig.~\ref{fig:spn1}, which directly represents
the computations expressed by Eqns.~\eqref{eq:subeq1}, with the coefficients
$P(\cdot|\cdot)$ omitted in Fig.~\ref{fig:spn1} for better visibility.
It is evident that the SPN does not clearly display the high level semantics due to conditional
independence of the graphical model. On the other hand, the SPN makes explicit the computational structure for efficient inference and encodes more compactly than GMs a class of relevant situations described next.

\paragraph{Introducing Contextual Independency. }

\begin{figure}[tb]
\subfloat[\label{fig:spnVsGm-1}]{\includegraphics[width=0.1\textwidth]{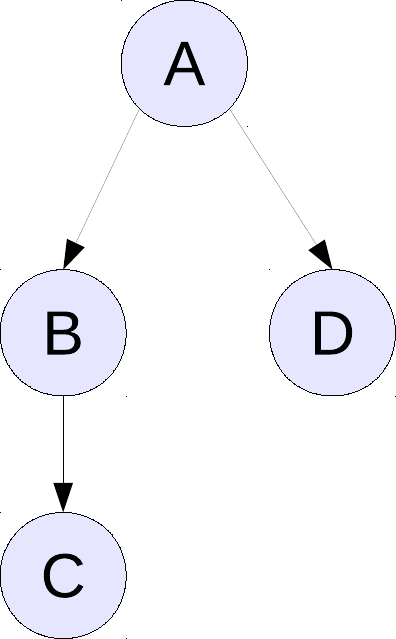} }\hfill{}\subfloat[\label{fig:spnVsGm-2}]{\includegraphics[width=0.15\textwidth]{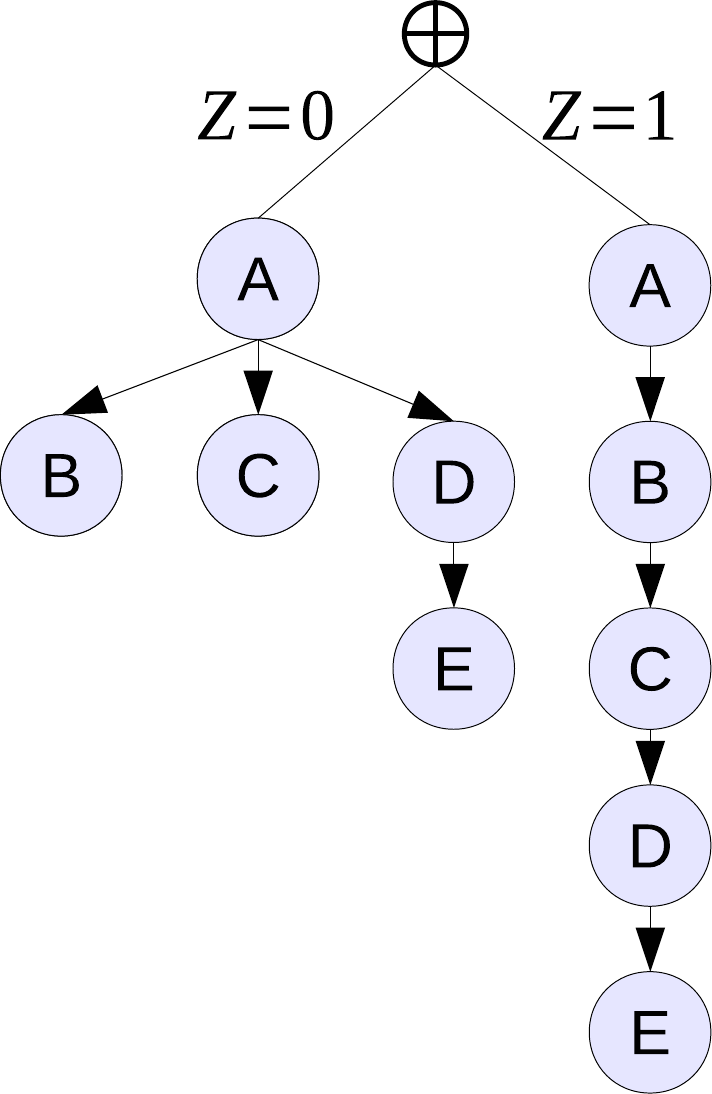} 

}\hfill{}\subfloat[\label{fig:spnVsGm-3}]{\includegraphics[width=0.35\textwidth]{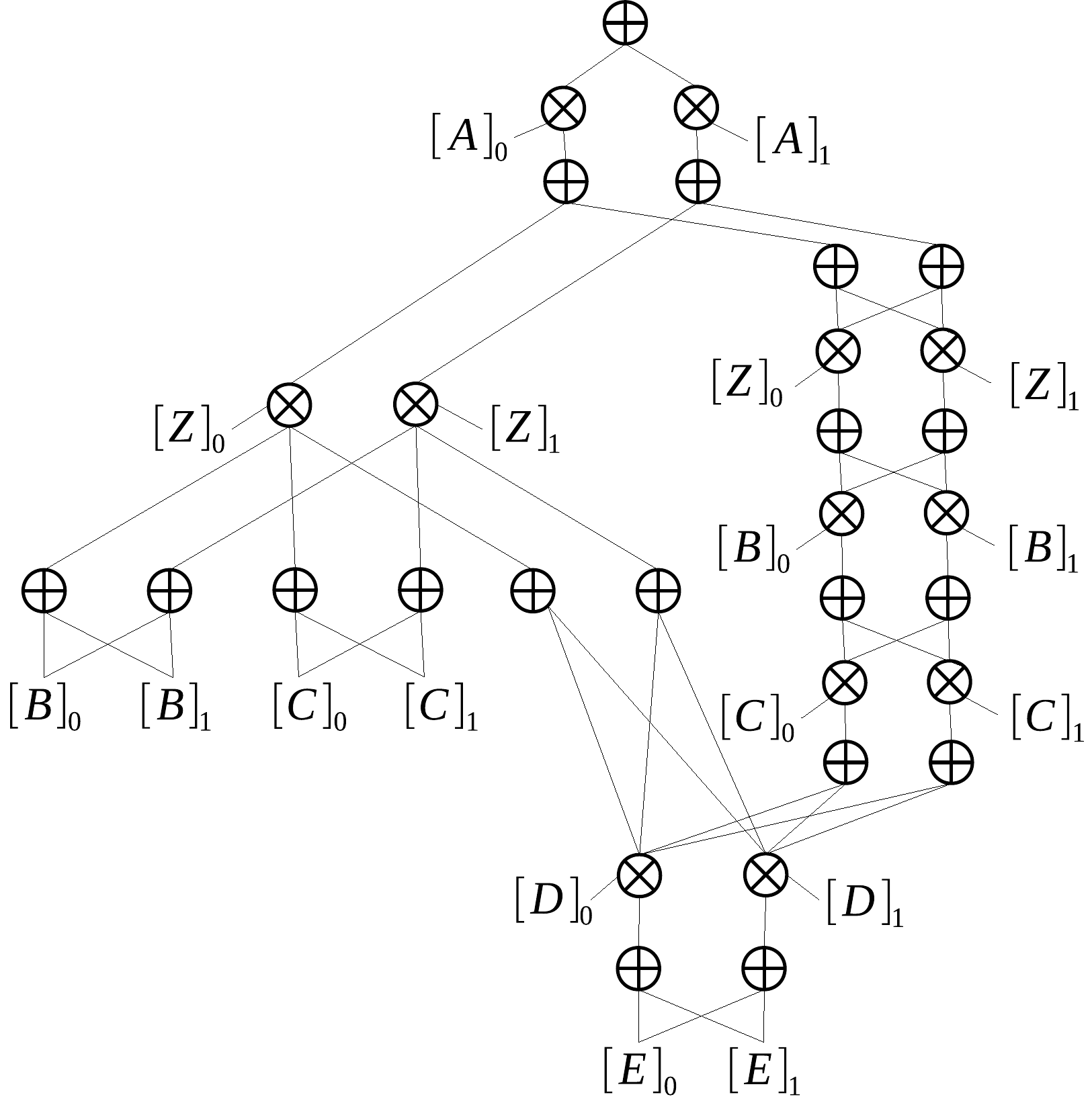} 

}\hfill{}\subfloat[\label{fig:spnVsGm-4}]{\includegraphics[width=0.2\textwidth]{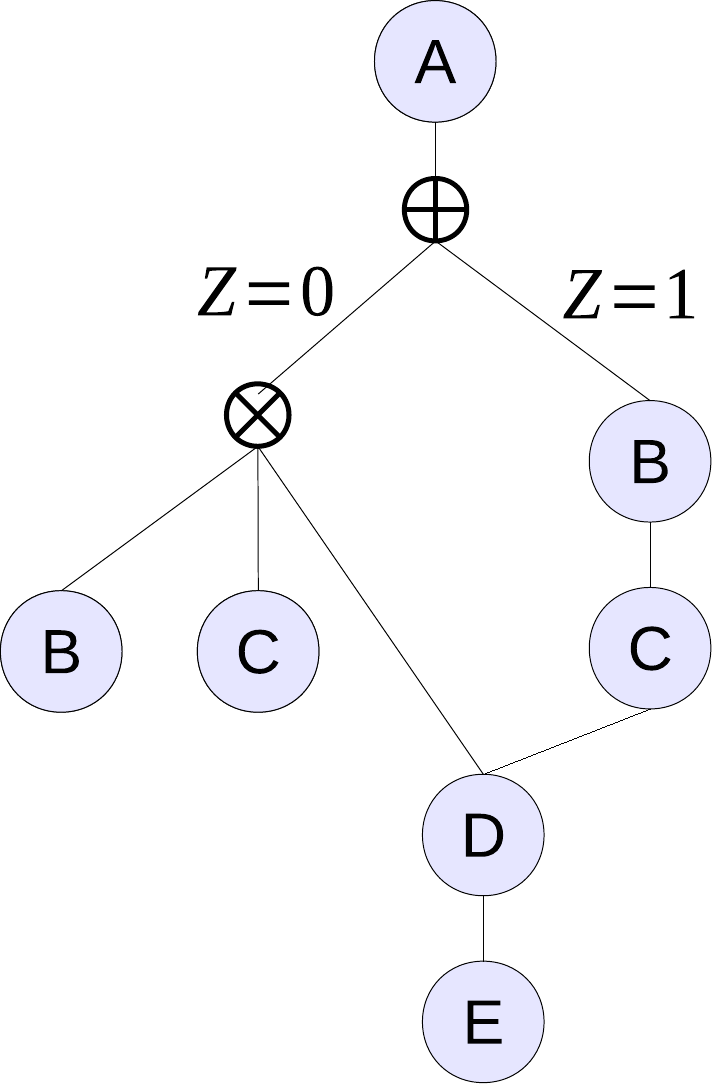} 

}\caption{The distribution in Eq. \ref{eq:ex2-1} represented (from left to
right) as Graphical Model (GM), as mixture of GMs, as Sum-Product
Network (SPN) and as Sum-Product Graphical Model (SPGM). }
\end{figure}

We consider a distribution in the form
\begin{subequations}\label{eq:ex2-1}
\begin{align}
P(A,B,C,D,E,Z)&=P(Z)P(A)P(B,C,D|A,Z)P(E|D)
\label{eq:ex2-1-a}
\intertext{with}
\label{eq:ex2-1-b}
P(B,C,D|A,Z)&=\begin{cases}
P(Z=0)P(B|A)P(C|A)P(D|A) & \text{if }Z=0, \\
P(Z=1)P(B|A)P(C|B)P(D|C) & \text{if }Z=1.
\end{cases}
\end{align}
\end{subequations}
Notice that different independency relations hold depending
on the value taken by $Z$: if $Z=0$, then $B,C$ and $D$ are conditionally
independent given $A$, whereas if $Z=1$, then they form a chain. We therefore say 
that this distribution exhibits \emph{context specific independence}
with \textit{context variable} $Z$. 

As in the example before, this distribution can be represented in different ways. Firstly, 
choosing a graphical model (Fig. \ref{fig:spnVsGm-1})
 requires to model $P(B,C,D|A,Z)$ as a single factor over $5$ variables,
since GMs cannot directly represent the \emph{if} condition of \eqref{eq:ex2-1-b}.\footnote{
A workaround involves factors with a complex structure, similar to SPNs, as done for instance in \citep{CaseFactorDiagrams}. Although this approach would be simple enough in the present example, it generally leads to a representation with the 
disadvantages of SPNs. 
See Section \ref{sec:Related-Architectures} for further discussion.}

Secondly, we may represent the distribution as a \textit{mixture}
of two tree-structured GMs (Fig.~\ref{fig:spnVsGm-2}),
\begin{subequations}
\begin{align}
P(B,C,D|A,Z) =\;
&P(Z=0)P(A)P(B|A)P(C|A)P(D|A)P(E|D)\\
+ &P(Z=1)P(A)P(B|A)P(C|B)P(D|C)P(E|D).
\end{align}
\end{subequations}
However, since some factors, here $P(A)$ and $P(E|B)$, appear in \textit{both} mixture components, this representation generally looses compactness, and computations for inference are unneccessarily \textit{replicated}.

Finally, we may also represent the distribution as  SPN (Fig.~\ref{fig:spnVsGm-3}) following the procedure outlined in the previous example. This represention allows to make explicit the
\emph{if} condition due to contextual independence and 
to share common parts in the two models components.
On the other hand, as in the example above, 
the probabilistic
relations which are easily readable in the other models, are hidden.
Furthermore, the SPN representation is considerably more convoluted than the
alternatives, and every state of every variable is explicitly represented.

\subsection{Sum-Product Graphical Models}\label{sec:Intro-SPGMs}
The previous section showed that SPNs conveniently represent context specific independence and algorithm structures for inference, whereas GMs directly display conditional independency through factorization. Several attempts were made in the literature to close this gap. We discuss related work in Section \ref{sec:Related-Architectures}.

Our approach to this problem is to introduce a new representation, called \textit{Sum-Product Graphical Model (SPGM)}, that
directly \emph{inherits} the favourable traits from both GMs and SPNs.
SPGMs can be seen as an extension
of SPN that, along with product and sum nodes as internal nodes,
also comprises \emph{variable nodes} which have the same role as
usual nodes in graphical models. Alternatively, SPGMs can be seen 
as an extension of directed GMs by adding sum and product nodes as
internal nodes. 
The SPGM representing the distribution \eqref{eq:ex2-1} is shown by Fig.~\ref{fig:spnVsGm-4}. It clearly reveals both 
the mixture of the two tree-structured subgraphs and the shared components.
Thus, SPGMs exhibit \textit{both} the expressiveness of SPNs and the high level semantics of GMs.

More generally, every SPGM implements a mixture of trees with shared subparts,
as in the above example:\footnote{An extension to mixtures of \textit{junction trees} \citep{Cowell-et-al-03} is straightforward but does not essentially contribute to the present discussion and hence is omitted.} \textit{Context variables} attached to sum nodes implement \emph{context
specific independence} (like $Z$ in Fig. \ref{fig:spnVsGm-4}) and select trees as model components to be combined. \emph{Conditional
independence} between variables, on the other hand,  can be read off from the graph due to D-separation \citep{Cowell-et-al-03}. 
SPGMs enable to represent in this way \emph{very large} mixtures, whose size grows exponentially
with the model size and are thus intractable if represented as a standard
mixture model. On the other hand, inference in SPGM has a worst
case complexity that merely is \emph{quadratic} in the SPGM size and effectively is quasi-linear 
in most practical cases.\footnote{More precisely, the complexity is $O(N M)$, where $N$ is the number of nodes and $M$ is the maximal number of parent nodes, of any node in the model.}

In addition, SPGMs generally provide an equivalent but more compact and high level representation of SPNs, with the additional property that the role of variables with respect to both contextual and conditional independency remains explicit. 
A compilation
procedure through message passing allows to convert  the SPGM 
(Fig. \ref{fig:spnVsGm-4}) into the equivalent SPN (Fig. \ref{fig:spnVsGm-3}) which directly supports computational inference.

\subsection{Structure Learning}

\begin{figure}
\hfill{}\subfloat[\label{fig:learnEx-1}]{\includegraphics[height=5cm]{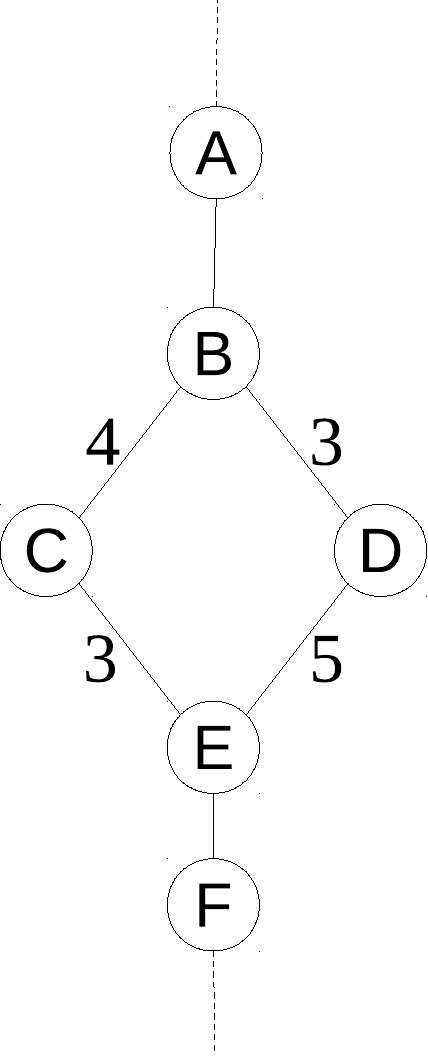} 

}\hfill{}\subfloat[\label{fig:learnEx-2}]{\includegraphics[height=5cm]{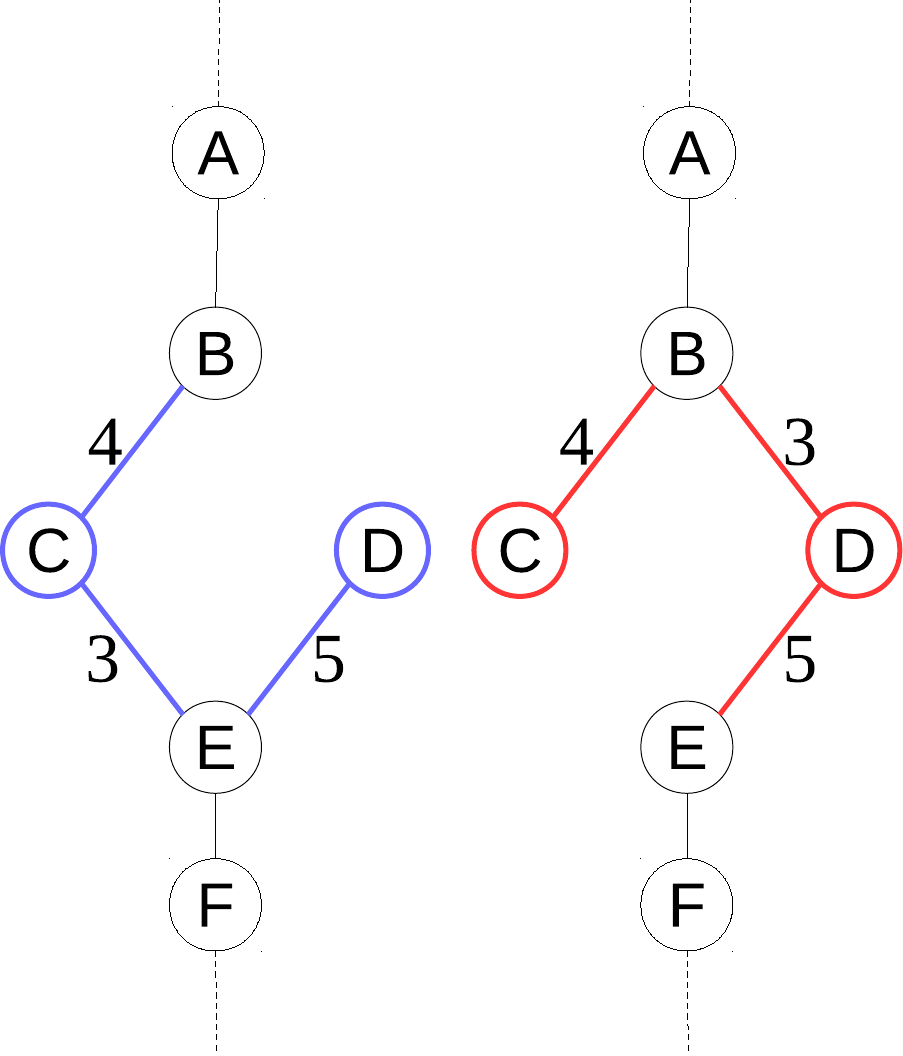} 

}\hfill{}\subfloat[\label{fig:learnEx-3}]{\includegraphics[height=5cm]{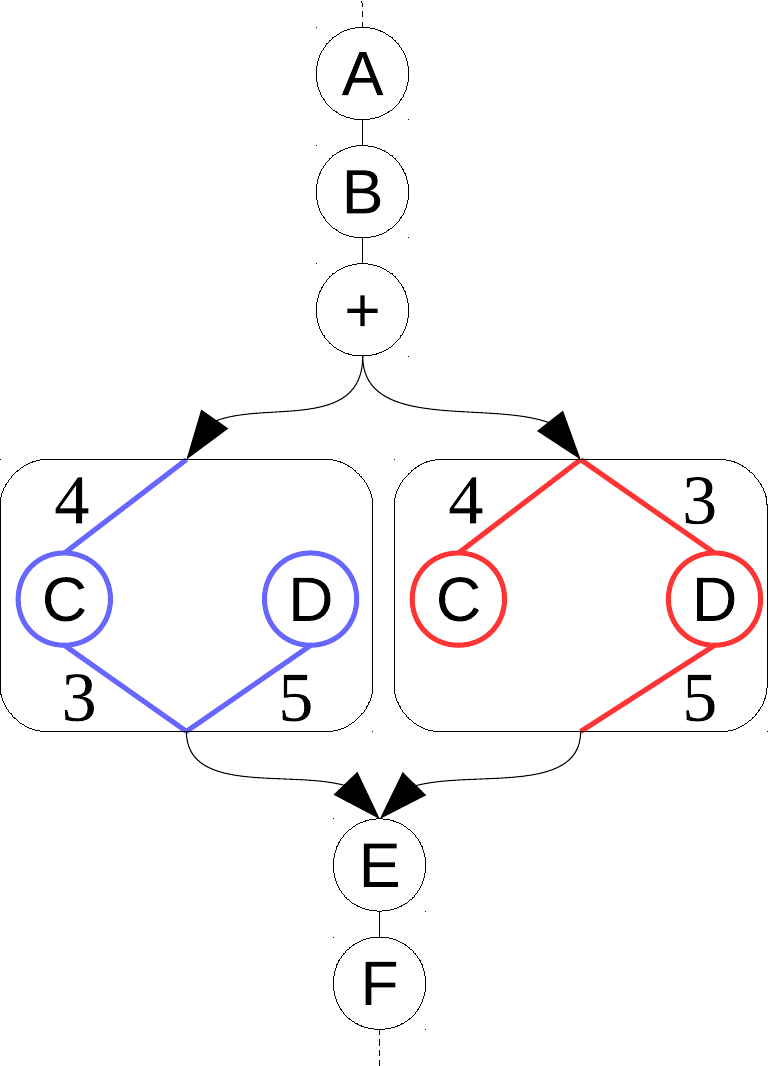} 

}\hfill{}

\caption{Sketch of the structure learning algorithm proposed
in this paper. (a) A weighted subgraph on which we compute the maximal spanning trees. (b)
Two maximal spanning trees of equal weight to included as mixture components into the SPGM. They differ only in a single edge. (c)
The mixture of the two trees represented by sharing all common
parts. 
}
\label{fig:mixMaxSteps12}
\end{figure}

Learning the structure of probabilistic models obviously is easier for models with tractable inference than for intractable ones, because any model parameter learning algorithm requires inference as a subroutine. 
For this reason, tractable probabilistic models and especially SPNs 
have been widely applied for density estimation  
\citep{SPNstructureLearning2013, Rooshenas14,Rahman16mergeSPN,Rahman16}. 
It it therefore highly relevant to provide and evaluate a structure learning
algorithm for SPGMs, that enable tractable inference as well. 

We introduce an algorithm that starts with fitting a single tree in the classical way \citep{ApproxDistrWithTrees} and iteratively
insert sub-optimal trees that have large weights (in terms of the mutual information of adjacent random variables) and share as many edges as possible with existing tree components. Each insertion is guaranteed not to decrease the global log-likelihood. As a result, all informative edges can be included into the model without compromising computational efficiency, because all shared components are evaluated only once. The former property is not true if a single tree is only fitted \citep{ApproxDistrWithTrees} whereas working directly with large tree mixtures \citep{LearningWithMixturesOfTrees} may easily lead to a substantial fraction of redundant computations.

Our approach is different from previous methods for learning the structure of SPNs, which mostly implement recursive partitioning of the variables into (approximately) independent clusters, to be represented by sum and product nodes \citep{SPNstructureLearning2013}. Clearly, our greedy method for learning the model structure and parameters based on \citep{LearningWithMixturesOfTrees} is only locally optimal as well, and focuses on the aforementioned statistical aspects that can be conveniently encoded by SPGMs. The results of a comprehensive experimental evaluation will be reported in Section \ref{sec:Empirical-Evaluation}.

\subsection{Related Work\label{sec:Related-Architectures}}

\begin{table}
\caption{Comparison of architectural properties discussed in Section \ref{sec:Related-Architectures}
for SPGMs and related architectures. We consider the following properties:
guaranteed tractable inference (TractInf); same
inference efficiency as SPNs (AsSPN); using exponential family factors with
no limitation on generality (ExpFam); high level representation of
conditional independence (CondInd); representation as a product of factors
as in graphical models (FactProd). \label{table:properties}}

\centering\vspace{5pt}
\resizebox{13 cm}{!}{%
\begin{tabular}{|l|l|l|l|l|l|}
\hline 
Model  & TractInf  & AsSPN  & ExpFam  & CondInd  & FactProd\tabularnewline
\hline 
\textbf{SPGM}  & \tableyes  & \tableyes  & \tableyes  & \tableyes  & \tableyes\tabularnewline
Graphical Models  &  &  & \tableyes  & \tableyes  & \tableyes\tabularnewline
Mixtures of Trees \citep{LearningWithMixturesOfTrees}  & \tableyes  &  & \tableyes  & \tableyes  & \tableyes\tabularnewline
Hierarchical MT \citep{Jordan94hierarchicalmixtures}  & \tableyes  &  & \tableyes  & \tableyes  & \tableyes\tabularnewline
Mix.Markov Model \citep{MixedMarkovModels}  &  &  & \tableyes  & \tableyes  & \tableyes\tabularnewline
Gates \citep{GatesMinka2009}  &  & \tableyes  & \tableyes  & \tableyes  & \tableyes\tabularnewline
\citep{Boutilier96context-specificindependence}  &  &  & \tableyes  & \tableyes  & \tableyes\tabularnewline
Case-Factor diagrams \citep{CaseFactorDiagrams}  & \tableyes  & \tableyes  &  &  & \tableyes\tabularnewline
BayesNets local Struct \citep{ABayesApprToLearnBayesNetsWithLocStr}  & \tableyes  & \tableyes  &  &  & \tableyes\tabularnewline
Learn.Efficient MarkovNets \citep{LearningEfficientMarkovNetworks}  & \tableyes  & \tableyes  &  &  & \tableyes\tabularnewline
\citep{PooleZhang2011}  & \tableyes  & \tableyes  &  &  & \tableyes\tabularnewline
Value Elimination \citep{ValueElimin}  & \tableyes  & \tableyes  &  &  & \tableyes\tabularnewline
SPN as Bayesian Nets \citep{SPNandBN}  & \tableyes  & \tableyes  &  &  & \tableyes\tabularnewline
And/Or Graphs \citep{AndOrDechter200773}  & \tableyes  & \tableyes  & \tableyes  &  & \tabularnewline
SPN \citep{SPN2011}  & \tableyes  & \tableyes  & \tableyes  &  & \tabularnewline
Arithmetic Circuits \citep{ArithmCircuitsAndNetworkPoly2}  & \tableyes  & \tableyes  & \tableyes  &  & \tabularnewline
CNets \citep{Rahman14Cnet}  & \tableyes  &  & \tableyes  & \tableyes  & \tabularnewline
\hline 
\end{tabular}} 
\end{table}

Table \ref{table:properties} lists and classifies prior work with a similar scope: introducing representations of probability distributions that fill to some extent the gap between GMs and SPNs. The caption of table \ref{table:properties} lists the  properties used to classify related work.

A major aspect of a SPGM is that it encodes a SPN through message passing. This will be made precise formally in Section \ref{sec:asSpn}. As a consequence, SPGMs relate to \textit{Arithmetic Circuits} \citep{ArithmCircuitsAndNetworkPoly2}, which differ from SPNs only in the way connection weights are represented, and to \textit{And/Or Graphs} \citep{AndOrDechter200773}, which are structurally equivalent to Arithmetic Circuits and thus also to SPNs. We refer to \citep[Section 7.6.1]{AndOrDechter200773} for a discussion of details. As discussed in Section \ref{sec:Intro-SPGMs}, SPGMs encode the computational structure for efficient inference like SPNs and related representations, but also preserve explicitly factorization properties of the underlying distribution due to conditional independence.



SPGMs can represent very large mixtures of trees. 
In this sense, SPGMs generalize \emph{Hierarchical
Mixture of Trees} \citep{Jordan94hierarchicalmixtures} by substituting
the OR-tree structure used to generate the trees in these models with
a general directed acyclic graph. 

SPGMs closely relate to \emph{Gates} \citep{GatesMinka2009}
and \emph{Mixed Markov Models} \citep{MixedMarkovModels}. These models
augment graphical models by a so-called gate unit that implements
context specific independence by switching edges on and off depending on the state of some context variables. In this respect, SPGMs may be regarded as Gates - see remark in Section \ref{subsec:Discussion} for corresponding technical aspects. 
However, unlike Gates and Mixed Markov Models, SPGMs guarantee \textit{tractable
exact} inference by construction. 

SPGMs related to several further methods that augment GMs by
factors with complex structure in order to represent context specific independence \citep{Boutilier96context-specificindependence,CaseFactorDiagrams,ABayesApprToLearnBayesNetsWithLocStr,LearningEfficientMarkovNetworks,ValueElimin}.
These approaches enable to represent product of factors like
graphical models. However, the additional model complexity due to contextual independence is simply encapsulated inside
the factors, based on models that are equivalent to SPNs and thus exhibit corresponding limitations (Section \ref{subsec:An-Illustrative-Example}). 
In particular, in connection with distributions that combine both conditional and contextual independence, the approaches have to resort to a low-level SPN-like representation. 
On the other hand, if simpler factors (such as with distributions from the exponential family) were used instead, the model would loose its expressivenes. 

\subsection{Contribution and Organization}

The contributions of this paper are twofold. First, we introduce SPGMs, which connect 
SPNs to Graphical Models in that they possess high level \emph{conditional
independence} semantics akin to Graphical Models. Moreover, they enable \emph{tractable} inference and to represent 
\emph{context specific} dependences and determinism as SPNs. 

The second contribution is a new structure learning approach that exploits
this connection by learning very large mixtures of quasi
optimal Chow-Liu trees with shared subparts. A comparative empirical
analysis show that this algorithm is competitive against state of
the art methods in density estimation. 

To empirically validate structure learning for SPGMs we tested the
learning performances in a density estimation task using $20$ real-life
datasets (Section \ref{sec:Empirical-Evaluation}). We obtain the
best results in $6$ datasets over $20$ comparing against $7$ state
of the art methods, and close performances in the other cases. The
results are particularly interesting given that our approach is novel,
but we compare to methods based on years of layered research and stemming
from a common approach. This novel approach thus opens up several
directions of improvement, inspired by techniques used in compared
algorithms. 


\paragraph{Structure of the Paper }

This paper is organized as follows. Section \ref{sec:Background}
contains notation and background on graphical models and SPNs. Section
\ref{sec:Sum-Product-Graphical-Models} describes SPGMs and discusses
their properties, analyzing the dual interpretation as mixture of
tree graphical models and as SPN. Section \ref{sec:Structure-Learning}
describes the proposed structure learning algorithm for SPGMs. Section
\ref{sec:Empirical-Evaluation} reports an extensive empirical evaluation
of SPGMs in density estimation.


\section{Background\label{sec:Background}}

\subsection{Directed Graphical Models\label{subsec:Directed-Graphical-Models}}

Let $\mathcal{G}=\left(\mathcal{V},\mathcal{E}\right)$ be a Directed Acyclic Graph (DAG) with vertex set $\mathcal{V}=\{1,2,\dotsc,N\}$ and edge set $\mathcal{E}$.
We associate to each vertex $s \in\mathcal{V}$ a discrete random variable $X_{s}$ taking values in the finite domain $\val\left(X_{s}\right)$, and $X = \{X_{s}\}_{s\in \mathcal{V}}$ denotes the set of all variables of the model, taking values in the Cartesian
product $\Delta(X) := \val\left(X_{1}\right)\times\val\left(X_{2}\right)\times\dots\times\val\left(X_{N}\right)$. We formally define the indicator variables already introduced in Section \ref{subsec:An-Illustrative-Example}, Eqs.~\eqref{eq:ex2}, \eqref{eq:subeq1}.
\begin{definition}[Indicator Variables]\label{def:assignmentIndVar}
Let $Y\subseteq X$ be a subset of variables to which the values $y\in\val(Y)$ are assigned: $X_{v} = y_{v},\; \forall X_{v} \in Y$. Based on the assignment $y$, we associated with every variable $X_{s} \in X$ and every value $i \in \Delta(X_{s})$ the \textit{indicator variable} $[X_{s}]_{i} \in \{0,1\}$ defined by
\begin{equation}
[X_{s}]_{i} = \begin{cases}
1 & \text{if $(X_{s} \in Y$ and $y_{s}=i$) or $(X_{s} \not\in Y)$}, \\
0 & \text{otherwise.}
\end{cases}
\end{equation}
\end{definition}
\noindent
We denote by $pa(s)$ the parents of $s$ in $\mathcal{G}$: $pa(s) = \{r \in \mathcal{V} \colon (r,s) \in \mathcal{E}\}$.

\vspace{0.2cm}
A \emph{Directed Graphical Model (Directed GM)} on a graph 
$\mathcal{G}$ 
comprises conditional probabilities $P_{s,t}\left(X_{t}|X_{s}\right)$ 
for every directed edge $\left(s,t\right)\in\mathcal{E}$
and unary probabilities $P_{r}\left(X_{r}\right)$ for each vertex $r\in\mathcal{V}$ with no parent, and encodes the  distribution
\begin{equation}\label{eq:P-directed-GM}
P(X)=\prod_{r\in\mathcal{V} \colon pa(r)=\emptyset}P_{r}\left(X_{r}\right)\prod_{\left(s,t\right)\in\mathcal{E}}P_{s,t}\left(X_{t}|X_{s}\right).
\end{equation}
A \emph{Directed Tree Graphical Model (Tree GM)} is a directed GM where the underlying graph $\mathcal{G}=\mathcal{T}$ is a rooted tree $\mathcal{T}$ with root $r$. Since each vertex $s$
has at most one parent $pa(s)$, the distribution \eqref{eq:P-directed-GM} reads 
\begin{equation}
T(X)=P_{r}\left(X_{r}\right)\prod_{s\in\mathcal{V} \colon pa(s)\ne\emptyset}P_{pa(s),s}\left(X_{s}|X_{pa(s)}\right).\label{eq:treeGm}
\end{equation}
Marginalization and Maximum a Posteriori (MAP) inference in general directed
GMs has a cost that is exponential in the treewidth of the triangulated graph
obtained by moralization of the original graph, and thus is intractable
for graphs with cycles of non trivial size \citep{Cowell-et-al-03,Diestel-06}.
However, in \emph{tree} GMs inference can be computed efficiently
with\emph{ message passing}. Let $Y\subseteq X$ 
be a set of observed variables with assignment $y\in\val(Y)$. 
Let $[X_{s}]_{j},\; s \in \mathcal{V},\; j \in \Delta(X_{s})$ denote indicator variables due to Definition \ref{def:assignmentIndVar}. Node $t$ sends a message $\mu_{t \rightarrow s ;j}$ to its
parent $s$ for each 
state $j\in\val(X_{s})$ given by
\begin{equation}
\mu_{t \rightarrow s ;j}=\sum_{k\in\val\left(X_{t}\right)}P_{s,t}\left(k|j\right)[X_{t}]_{k}\prod_{(t,q) \in \edg}\mu_{q \rightarrow t;k}.\label{eq:treeMsgPass}
\end{equation}
Setting $C = X \setminus Y$ and $x = (y,c)$, marginal probabilities 
$T(Y=y)=\sum_{c \in \Delta(C)} T(Y=y, C=c)$
can be computed using the distribution \eqref{eq:treeGm} by first setting the indicator variables according
to the assignment $y$ (Definition \ref{def:assignmentIndVar}), then passing messages
for every node in reverse topological order (from leaves to the root), and
finally returning the value of the root message. MAP queries are computed
in the very same way after substituting summmation with the $\max$ operation in Eq.~\eqref{eq:treeMsgPass}.
Since message passing in trees only requires computing one message per each node, the procedure has complexity $O(|\mathcal{V}|\Delta_{max}^{2})$, where $\Delta_{max} = \max\{|\Delta(X_{s})| \colon s \in \mathcal{V}\}$
is the maximum domain size. As a consequence, tree GMs enable \emph{tractable}
inference.

Graphical models conveniently encode
conditional independence properties
of a distribution.
Conditional independence between variables in GMs are induced from
the graph by $D$-separation (see, e.g., \cite{Cowell-et-al-03}). In
the case of tree graphical models, D-separation becomes
particularly simple: if the path between variables $A$ and $B$ contains
$C$, then $A$ is conditionally independent from $B$ given $C$.

It is well-known, however, that another form of independence are not covered well by
the GMs: independence that
only holds in certain contexts, i.e.~depending on the assignment
of values to a specific subset of so-called context variables. 
\begin{definition}[Contextual Independence]\label{defCsi} 
Variables $A$ and $B$ are said to be \textit{contextually
independent} given $Z$ and \textit{context} $z\in\val(Z)$ if $P(A,B|Z=c)=P(A|Z=z)P(B|Z=z)$. 
\end{definition} 
\begin{remark}\label{rem:contextual-independence}
Notice that \textit{conditional} independence is a special case of \textit{contextual} independence in which $P(A,B|Z=c)=P(A|Z=z)P(B|Z=z)$ would hold for \emph{all} $z\in\val(Z)$. By contrast, \textit{contextual} independence assures that this property only holds for a subset of value assignments that constitute the so-called \textit{context}. In particular, different independences can hold for different values $z$ -- see Eq.~\eqref{eq:ex2-1-b} for an illustration. 
We refer to \cite{Boutilier96context-specificindependence} for an in-depth discussion of contextual independence.
\end{remark}
\noindent
Encoding contextual independence mainly motivates the model class of sum-product networks formally introduced in Section \ref{sec:intro-SPN}.

\paragraph{Mixtures of Trees.} 
A \emph{mixture of trees} is a distribution in the form $P(X)=\sum_{k=1}^{K}\lambda_{k} T_k\left(X \right)$
where $\{T_{k}\left(X \right)\}_{k=1}^K$ are directed tree GMs and $\{\lambda_{k}\}_{k=1}^K$ are
real non-negative weights satisfying $\sum_{k=1}^{K}\lambda_{k}=1$.
Inference in mixture models involves taking the weighted sum of the results of inference in each tree. Hence it has $K$ times the cost of inference in a single tree in the mixture. 

A mixture model can also be expressed through a hidden variable $Z$ with $\val(Z)=\{1,2,\dots,K\}$ by  writing: $P(X,Z)=\prod_{k=1}^{K}\left( \lambda_k T_k\left(X \right) \right)^{\delta(Z=k)}$. Then, it holds that  $P(X)=\sum_{z\in \val(Z)}P(X,Z)$, $P(Z=k)=\lambda_k$ and $P(X|Z=k)=\lambda_k T_k\left(X \right) $.
Note that different values of $Z$ entail different independences due to different tree structures. It follows that mixtures of trees represent context-specific independence with context variable $Z$. 
 
However, the family of conditional independences entailed by mixture models is limited to using a \emph{single} context variable $Z$, and to the selection of different models defined on the \emph{entire} set  $X$ for each value of $Z$. In contrast, the model class of sum-product networks  to be introduced in the next section, allows to model contextual independences depending on \emph{multiple} context variables that only affect a \emph{subset} of $X$ -- see, e.g., the example in Section \ref{subsec:An-Illustrative-Example}.

\subsection{Sum-Product Networks\label{subsec:Sum-Product-Networks}}\label{sec:intro-SPN}

\textit{Sum-Product Networks (SPN)} were introduced in \citep{SPN2011}. They are
closely related to Arithmetic Circuits \citep{Darwiche2003}. We adopt the definition of \emph{decomposable} SPNs advocated by \citep{SPNstructureLearning2013}. The expressiveness of these models was shown by \citep{Peharz2015phd} to be equivalent to the expressiveness of non-decomposable SPNs.   

\begin{definition}[Sum-Product Network (SPN)]\label{def:SPN} Let $X=\{X_{1},\dotsc,X_{N}\}$ be a collection of random variables, and let $X=X^{1}\cup X^{2}\cup \dotsb \cup X^{K}$ be a partition of $X$. A \textit{Sum-Product Network (SPN)} $S(X)$
is recursively defined and constructed by the following rule and operations: 
\begin{enumerate}
\item An indicator variable $[X_{s}]_{j}$ is a SPN $S(\{X_{s}\})$.\label{enu:An-indicator-variable}
\item The \textit{product} $\prod_{k=1}^{K}S_{k}(X_{k})$ is a SPN, with the SPNs $\{S_{k}(X_{k})\}_{k=1}^{K}$ as factors.\label{enu:The-product-}
\item The \textit{weighted sum} $\sum_{k=1}^{K}w_{k}S_{k}(X)$ is a SPN, with the SPNs $\{S_{k}(X)\}_{k=1}^{K}$ as summands and non-negative weights $\{w_{k}\}_{k=1}^{K}$.\label{enu:The-weighted-sum} 
\end{enumerate}
\end{definition}
\begin{remark}\label{rem:evaluation-S-y}
Due to Definition \ref{def:assignmentIndVar}, indicator variables that form SPNs require the specification of an assignment $y$ of a subset of variables $Y \subset X$, in order to be well-defined. Such assignments are specified in connection with the \textit{evaluation} of a SPN, denoted by 
\begin{equation}\label{eq:def-Sy}
S(y):=S(X\setminus Y, Y=y).
\end{equation}
\end{remark}
\begin{example}
The SPN shown by Fig.~\ref{fig:spn1} displays the operations due to Definition \ref{def:SPN}. Overall, it represents the operations of the sum-product message passing procedure  \eqref{eq:treeMsgPass}) with respect to the graphical model of Fig.~\ref{fig:gm1}, where the indicator variables introduce  evidence values (measurements, observations). Notice that indicator variables are the only variables in this SPN, and that the probabilities $P_s , P_{s,t}$ define weights attached to the sum nodes. 
\end{example}

SPNs are evaluated in a way similar to message passing in tree graphical models (Eq. \eqref{eq:treeMsgPass}). 
An evaluation $S(y)$ due to \eqref{eq:def-Sy} is computed by first
assigning indicator variables in $S$ according to $y$ by Definition \ref{def:assignmentIndVar}, then evaluating nodes in inverse
topological order (leaves to root) and taking the value of the root node of $S$. MAP queries are computed
in the very same way after substituting sum nodes with $\max$ nodes. 

\cite{SPN2011} shows that evaluating $S(y)$ corresponds to computing marginals in some
valid probability distribution $P(X)$, namely: $S(y)=P(Y=y)=\sum_{x_{\setminus y}\in \val(X \setminus Y)} P(x)$, where $x_{\setminus y}$ denotes assignments to variables in the set $X \setminus Y$.
In addition,  evaluating $S(y)$ involves evaluating each node in the DAG, i.e.~inference at $O(\mathcal{E})$ time and memory cost 
 is always \emph{tractable}.

The recursive Definition \ref{def:SPN} enables to build up SPNs by iteratively composing
SPNs using the operations \ref{enu:The-product-} and \ref{enu:The-weighted-sum},
starting from trivial SPNs as leaves in the form of indicator variables (rule \ref{enu:An-indicator-variable}). Clearly, it is
not always intuitive to design a SPN by hand, which involves thinking of the product nodes as factorizations and of the sum nodes as mixtures of the underlying sub-models. For some particular applications, the SPN structure was designed and fitted in this way to the particular distribution at hand, cf.~\citep{SPN2011,SPNlangMod2014,SPnforActRecogn2015}). Yet, more generally, SPN approaches involve
some form of automatic learning the structure of the SPN from given data \citep{SPNstructureLearning2013}. 
Our approach to learning the structure of a SPN is described in Section \ref{sec:Structure-Learning}.

Besides tractability, the main motivation to consider the model class of SPNs is their ability to represent contextual independences more expressively than mixture models.  A detailed analysis of the role of contextual independence in SPNs is provided in Section \ref{sec:Sum-Product-Graphical-Models}.
Intuitively, the expressiveness of SPNs descends from interpreting sum nodes as \emph{mixtures} of the children SPNs, which allows to create a hierarchy of mixture models and thus a hierarchy of contextual independences, by using hidden mixture variables. In addition, sharing the same children SPNs in different sum nodes allows to limit the effect of contextual independences to a subsection of the model, as shown in the example in Section \ref{subsec:An-Illustrative-Example}.

\section{Sum-Product Graphical Models\label{sec:Sum-Product-Graphical-Models}}

This section introduces \textit{Sum-Product Graphical Models (SPGMs)} and discusses their properties. 
We provide an interpretation of the model as a mixture of tree GMs in Section \ref{sec:asMixTrees}, and as a high-level representation of SPNs in Section \ref{sec:asSpn}.

\subsection{Definition}

The first step in describing SPGMs is to define the type of nodes
that appear in the underlying graph. 

\begin{definition}[Sum, Product and
Variable Nodes]\label{def:nodeTypes} Let $X$ and $Z$ be disjoint sets of discrete variables,
and let $\mathcal{G}=(\mathcal{V},\mathcal{E})$ be a DAG.
\begin{itemize}
\item The basic nodes $s\in\mathcal{V}$ of a SPGM  are called \textit{SPGM Variable Node (Vnode)} and associated with a variable $X_{s}\in X$, They are graphically represented as
a circle having $X_{s}$ as label (Fig.~\ref{fig:ESPNexample}, left). 
\item $s\in\mathcal{V}$ is called \textit{Sum Node}, if it represents the corresponding operation indicated by the
symbol $\oplus$. A Sum Node can be \emph{Observed}, 
in which case it is associated to a variable $Z_{s}\in Z$ and represented
by the symbol $\oplus Z_{s}$. 
\item $s\in\mathcal{V}$ is called \textit{Product Node}, if it represents the corresponding operation indicated by the symbol $\otimes$.
\end{itemize}
\end{definition} 

\noindent
In what follows, variables $Z$ take the role of \emph{context variables} according to Definition \ref{defCsi}.


\begin{definition}[Scope of a node] Let $\mathcal{G} = (\mathcal{V},\mathcal{E})$
be a rooted DAG with nodes as in Definition \ref{def:nodeTypes},
and let $s\in\mathcal{V}$. 
\begin{itemize}
\item The \emph{scope} of $s$ is the set of all variables associated to
nodes in the sub-DAG rooted in $s$. 
\item The $X$\emph{-scope} of $s$ is the set of all variable associated
to Vnodes in the sub-DAG rooted in $s$. 
\item The $Z$\emph{-scope} of $s$ is the set of variables associated to
Observed Sum Nodes in the sub-DAG rooted in $s$. 
\end{itemize}
\end{definition}
\begin{example}
The scope of the Sum Node associated to $Z_{2}$ in Fig.~\ref{fig:spgm1} is $\{D,E,F,Z_{2}\}$, its $Z$-scope is $\{Z_{2}\}$,
and its $X$-scope is $\{D,E,F\}$.
\end{example}
\noindent
Finally, we define the set of ``V-parents'' of a Vnode $s$, which intuitively are 
the closest Vnode ancestors of $s$. 

\begin{definition}[Vparent]\label{def:Vparent} The \emph{Vparent set} $\vpa\left(s\right)$ of a Vnode $s$ is
the set of all $r\in\mathcal{V}$ such that $r$ is a \emph{Vnode},
and there is a directed path from $r$ to $s$ that does not include
any other Vnode. 
\end{definition}

\noindent
With the definitions above we can now define SPGMs.

\begin{definition}[SPGM]\label{def:ESPN} A \textit{Sum-Product Graphical Model
(SPGM)} $\spgm\left(X,Z|\mathcal{G}, \{P_{st}\}, \{W_{s}\}, \{Q_{s}\} \right)$ or more shortly, $\spgm\left(X,Z\right)$ or even $\spgm$, is a rooted DAG $\mathcal{G}=(\mathcal{V},\mathcal{E})$
where nodes can be Sum, Product or Vnodes as in Definition \ref{def:nodeTypes}.
The SPGM is governed by the following parameters:
\begin{enumerate}
\item Pairwise conditional probabilities $P_{st}\left(X_{t}|X_{s}\right)$
 associated to each Vnode $t \in \mathcal{V}$ and each Vparent $s\in\vpa\left( t \right)$.
\item Unary probabilities $P_{s}\left(X_{s}\right)$  associated to each
Vnode $s\in \mathcal{V} \colon \vpa(s)=\emptyset$.
\item Unary probabilities $W_{s}(k)$ for $k=\{1,2,...,|ch(s)|\}$ associated to each non-Observed Sum
Node $s$, with value $W_{s}(i)$ associated to the edge between $s$ and its $i$-th child (assuming any order has been fixed). 
\item Unary probability $Q_{s}(Z_{s})$ s.t. $\val(Z_s)=\{1,2,...,|ch(s)|\}$ and associated to each Observed Sum
Node $s$, with value $Q_{s}(i)$ being associated to the edge between
$s$ and its $i$-th child. 
\end{enumerate}
In addition, each node $s\in\mathcal{V}$ must satisfy the following conditions: 
\begin{enumerate}\setcounter{enumi}{4}
\item \label{enu:sgm1} If $s$ is a \textit{Vnode} (associated to variable $X_s$), then $s$ has \textit{at most} one child $c$,
and $X_{s}$ does not appear in the scope of $c$. 
\item \label{enu:sgm3} If $s$ is a \textit{Sum} Node, then $s$ has \textit{at least} one child,
and the scopes of all children are the \textit{same} set. If the Sum Node
is Observed (hence associated to variable $Z_s$), then $Z_{s}$ is not in the scope of any child.
\item \label{enu:sgm2} If $s$ is a \textit{Product} Node, then $s$ has \textit{at least} one
child, and the scopes of all children are \textit{disjoint} sets. 
\end{enumerate}
\end{definition}
An example SPGM is shown in Fig. \ref{fig:ESPNexample}, left. Note
that the closeness with both SPNs and GMs, to be further discussed later, can be already seen from the definition: the last three conditions in definition
are closely related to SPN conditions (Definition \ref{def:SPN}),
whereas the usage of pairwise and unary probabilities in 1.-4. above connects SPGMs
to graphical models. 

\begin{figure}
\subfloat[\label{fig:spgm1}]{ \includegraphics[height=5cm]{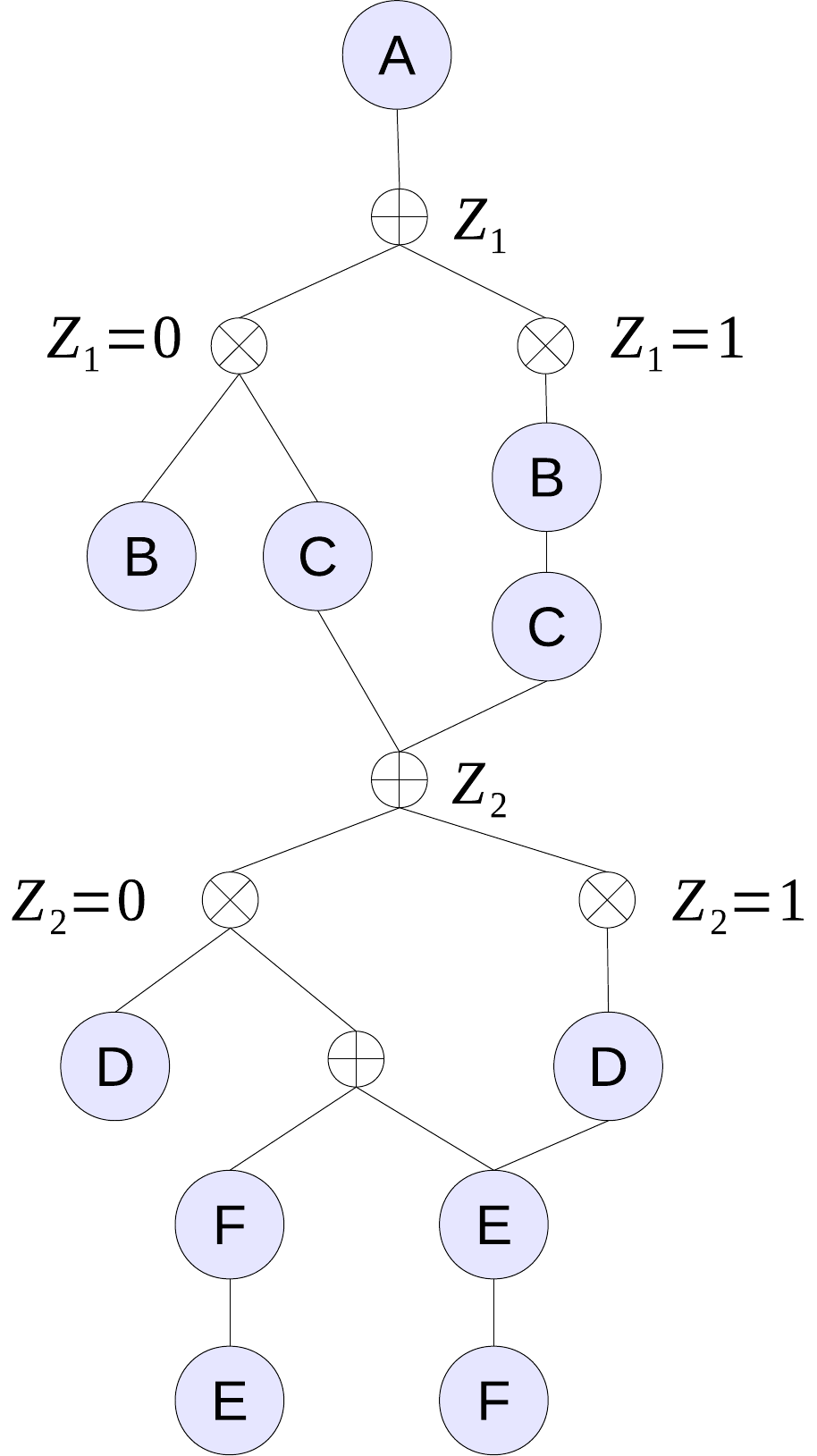}
}\hfill{}\subfloat[\label{fig:spgm2}]{ \includegraphics[height=5cm]{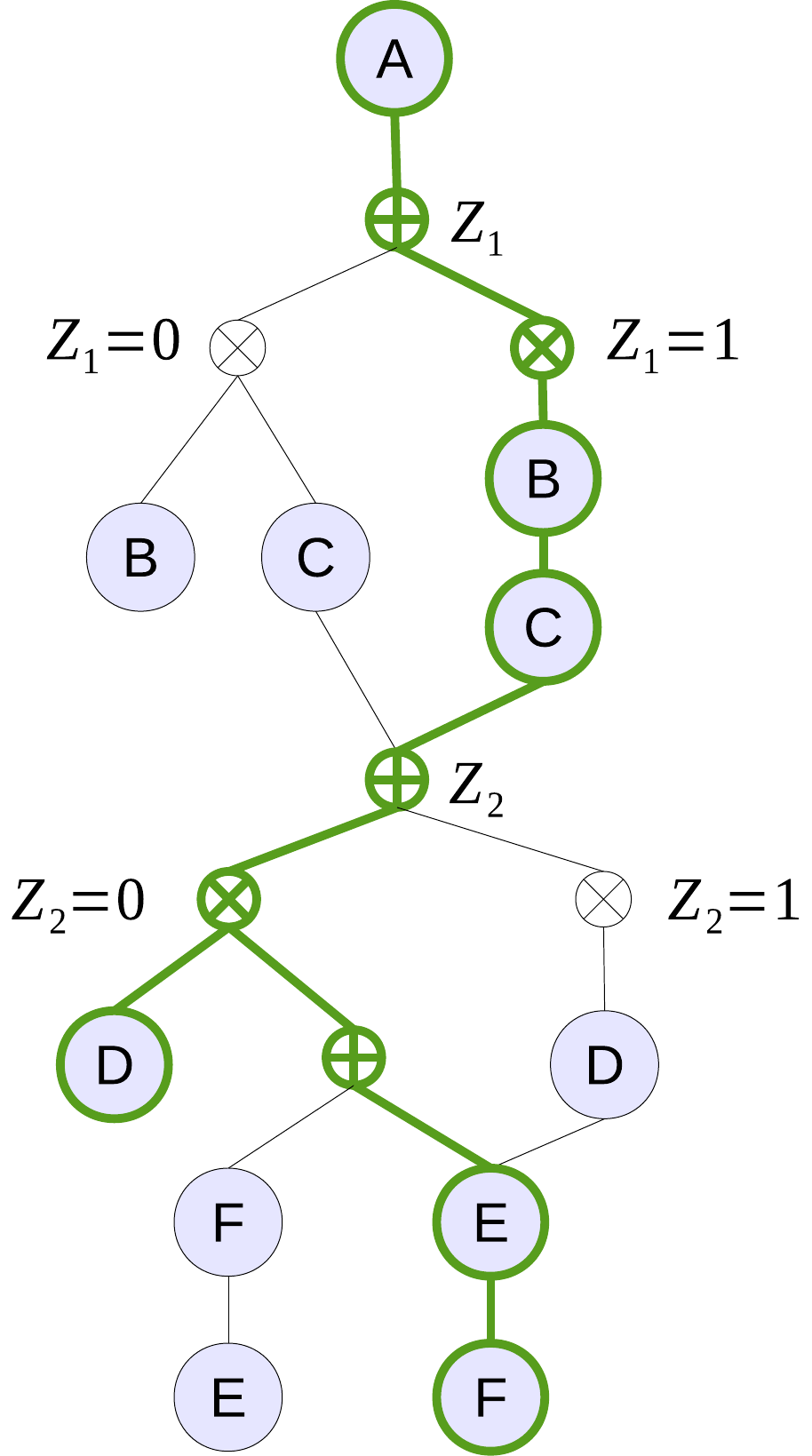}
}\hfill{}\subfloat[\label{fig:spgm3}]{ \includegraphics[height=5cm]{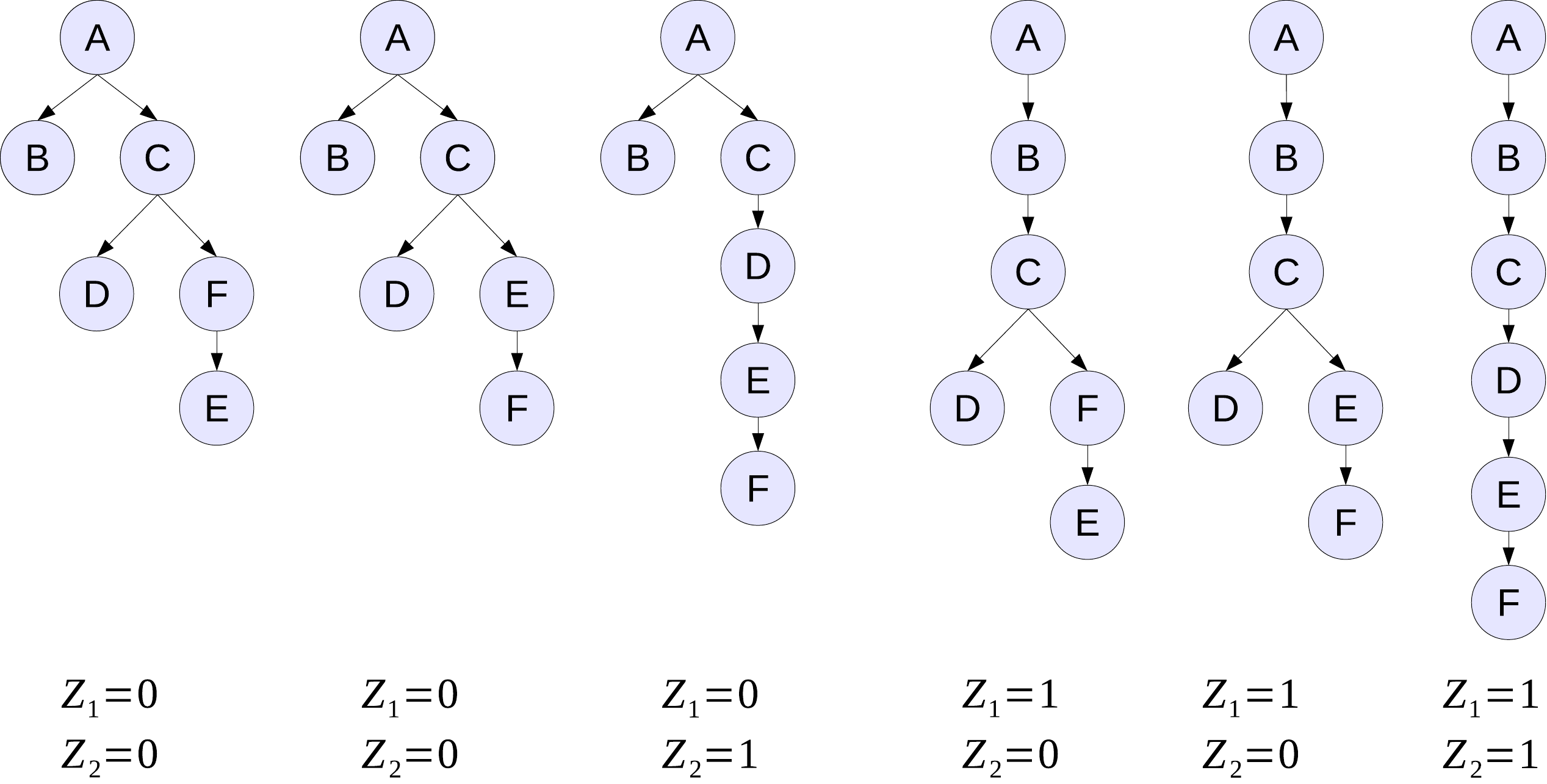}
}

\caption{ Sum-Product Graphical Model Example. 
(a) A SPGM $\spgm\left(X,Z\right)$
with $X=\left\{ A,B,C,D,E,F\right\} ,Z=\left\{ Z_{1},Z_{2}\right\} $.
(b) A subtree of $\spgm$ (in green). 
(c) All subtrees of $\spgm$ are represented as graphical models and
corresponding context variables. \label{fig:ESPNexample}}
\end{figure}

\subsection{Message Passing in SPGMs}

We now define a message passing protocol used to evaluate SPGMs, which
conforms to the way how SPGMs represent conditional and contextual
independence efficiently. The following definition refers to Definitions \ref{def:Vparent} and \ref{def:ESPN}.

\begin{definition}[Message passing
in SPGMs]\label{def:Message-passing-in} Let $s\in\mathcal{V},t\in\mathcal{V}$ and let $ch(s)_{k}$ denote the
$k$-th child of $s$ in a given order. Node $t$ sends a message
$\mu_{t \rightarrow s;j}$ to each Vparent $s\in\vpa\left(t\right)$ and for each
parent state $j\in\val(X_{s})$ according to the following rules: 
\begin{subequations}\label{eq:messages}
\begin{eqnarray}
\mu_{t \rightarrow s;j}= & \sum_{k=1}^{|ch(t)|}[Z_{t}]_{k}Q_{t}(k)\mu_{ch(t)_{k} \rightarrow s;j} & \text{$t$ is a Sum Node, Observed}\label{eq:msgSumNode}\\
\mu_{t \rightarrow s;j}= & \sum_{k=1}^{|ch(t)|}W_{t}(k)\mu_{ch(t)_{k} \rightarrow s;j} & \text{$t$ is a Sum Node, not-Observed}\label{eq:msgSumUnobs}\\
\mu_{t \rightarrow s;j}= & \prod_{q\in ch\left(t\right)}\mu_{q \rightarrow s;j} & \text{$t$ is a Product Node}\label{eq:msgProd}\\
\mu_{t \rightarrow s;j}= & \sum_{k\in\val\left(X_{t}\right)}P_{s,t}\left(k|j\right)[X_{t}]_{k}\mu_{ch(t) \rightarrow t;k} & \text{$t$ is a Vnode}\label{eq:msgPnode}
\end{eqnarray}
If $\vpa\left(s\right)$ is empty, top level messages
are computed as: \label{eq:messagesTop}
\begin{eqnarray}
\mu_{t \rightarrow	root}= & \sum_{k=1}^{|ch(t)|}[Z_{t}]_{k}Q_{t}(k)\mu_{ch(t)_{k} \rightarrow root} & \text{\text{$t$ is a Sum Node, Observed}}\\
\mu_{t \rightarrow	root}= & \sum_{k=1}^{|ch(t)|}W_{t}(k)\mu_{ch(t)_{k} \rightarrow root;j} & \text{$t$ is a Sum Node, not-Observed}\\
\mu_{t \rightarrow	root}= & \prod_{q\in ch\left(t\right)}\mu_{q \rightarrow root;j} & \text{$t$ is a Product Node}\\
\mu_{t \rightarrow	root}= & \sum_{k\in\val\left(X_{t}\right)}P_{s}\left(k\right)[X_{t}]_{k}\mu_{ch(t) \rightarrow t;k} & \text{$t$ is a Vnode}\label{eq:eq:msgPnodeRoot}
\end{eqnarray}

Vnodes at the leaves send messages as in Eqns.~\eqref{eq:msgPnode} and \eqref{eq:eq:msgPnodeRoot} after substituting the
incoming messages by the constant $1$. 

\end{subequations} \end{definition}

Note that that messages are only sent to Vnodes (or to a fictitious ``root''
for top level nodes), and no message is sent to Sum and Product Nodes.
Notice further that Vnode messages resemble message passing in tree GMs
(Eq. \eqref{eq:treeMsgPass}), which is the base for our subsequent
interpretation of SPGMs as graphical model.

\begin{definition}[Evaluation of $\spgm\left(X,Z\right)$]\label{def:SPGM-marg} 
Let $Y\subseteq X\cup Z$ denote evidence variables with assignment $y\in\val\left(Y\right)$. The evaluation of a SPGM $\spgm$ with assignment
$y$, written as $\spgm\left(y\right)$ (cf.~Remark \ref{rem:evaluation-S-y} and \eqref{eq:def-Sy}), is obtained by setting
the indicator variables accordingly (Definition \ref{def:assignmentIndVar}),
followed by evaluating messages for each node from the leaves to the root due to Definition
\ref{def:Message-passing-in}, and then taking the value of the message
produced by the root of $\spgm$.

\end{definition}

\begin{proposition}\label{proposition:linearCost}The evaluation
of a SPGM $\spgm$ has complexity $O(|\mathcal{V}|M|\val_{max}|^{2})$, where
$M$ is the maximum number of Vparents for any node in $\spgm$, 
and $|\val_{max}| = \max\{\Delta(X_{s}) \colon s \in \mathcal{V}\}$ is the maximum domain size for any variable in
$X$.

\end{proposition}
\begin{proof}
Every message is evaluated exactly once. Each of the
$|\mathcal{V}|$ nodes sends at most $M$ messages (one for each
Vparent), and each message has size $|\val_{max}|^{2}$ (one value
per every state of sending and receiving node). 
\end{proof}

\subsection{Interpretation of SPGMs as Graphical Models}\label{sec:asMixTrees}

In this section, we consider and discuss SPGMs as  probabilistic models. We show that SPGMs encode large mixtures of trees with shared subparts and
provide a high-level representation of both conditional and contextual
independence through D-separation.

\subsubsection{Subtrees}

We start by introducing subtrees of SPGMs and their properties. 

\begin{definition}[Subtrees of SPGMs]\label{def:ESPNsubn}

Let $\spgm$ be a SPGM. A \emph{subtree} $\tau(X,Z)$ (or more shortly $\tau$) is a SPGM defined on a subtree of the DAG $\mathcal{G}$ underlying $\spgm$ (cf.~Def.~\ref{def:ESPN}), that is recursively constructed based on the root of $\spgm$ and the following steps: 
\begin{itemize}
\item If $s$ is a \textit{Vnode or a Product Node}, then include in $\tau$ all children
of $s$ and edges formed by $s$ and its children. Continues this process for all included nodes.
\item If $s$ is a Sum Node, then include in $\tau$ only the $k_{s}$-th child
and the corresponding connecting edge, where the choice of $k_{s}$
is arbitrary. Continue this process for all included nodes.
\end{itemize}
We denote by $\enctr(\spgm)$ the set of all subtrees of $\spgm$.
\end{definition}
\begin{example}
One of the subtrees of the SPGM depicted in Fig.~\ref{fig:spgm1}
 is shown by Fig.~\ref{fig:spgm2}. 
\end{example}

\begin{definition}[Subtrees $\tau$ and indicator variable sets $z_{\tau}$]\label{def:ESPNsubn-3}

Let $\tau\in\enctr(\spgm)$ be a subtree of $\spgm$. The symbol $z_{\tau}$
denotes the set of all indicator variables associated to Observed
Sum Nodes and their corresponding state in the subtree. Specifically, if
the $k_{s}$-th child of an Observed Sum Node $s$ is included in the tree,
then $[Z_{s}]_{k_{s}}\in z_{\tau}$. 
\end{definition}

\begin{example}
The set $z_{\tau}$ for the subtree in Fig. \ref{fig:spgm2} is $\{[Z_{1}]_{1},[Z_{2}]_{0}\}$. 
\end{example}


\begin{definition}[Context-compatible subtrees]\label{def:ESPNsubn-2}
Let $Y\subseteq Z$  be a subset of context variables with assignment $y\in \val(Y)$, and let $[y]$ denote the set of indicator variables corresponding to $Y=y$. 
The set of subtrees compatible with context $Y=y$, written as $\enctr(\spgm|y)$, is the set of all subtrees $\tau\in\enctr(\spgm)$
such that $[y] \subseteq z_{\tau}$.
\end{definition}

\begin{example}
The set of subtrees $\enctr(\spgm|Z_{1}=0, Z_{2} = 0 )$ for the SPGM in in Fig. \ref{fig:ESPNexample} is composed by subtree $\tau$, shown in Fig. \ref{fig:spgm2}, and the subtree obtained by modifying $\tau$ through choosing the alternate child of the lowest sum node.
\end{example}

\noindent We now state properties of subtrees that are essential for the subsequent
discussion. 

\begin{proposition}\label{prop:subIsTree}
Any subtree $\tau \in \enctr(\spgm)$ is a tree SPGM. 
\end{proposition} 

\begin{proof} Only Product Nodes in $\tau \in \enctr(\spgm)$ can have multiple children,
since Vnodes have a single child by Definition \ref{def:ESPN}, case \ref{enu:sgm1},
and Sum Nodes have a single child in $\tau$ by
Definition \ref{def:ESPNsubn}. Children of Product Nodes have disjoint
graphs by Definition \ref{def:ESPN}, case \ref{enu:sgm2}. Therefore
$\tau$ contains no cycles. A rooted graph with no cycles is a tree.
\end{proof}

\begin{proposition}\label{prop:exp}

The number $|\enctr(\spgm)|$ of subtrees of $\spgm$ grows as $O(\exp(|\mathcal{E}|))$.

\end{proposition} 
\begin{proof} See Appendix \ref{Proof-of-prop:exp}. 
\end{proof}

\begin{proposition}\label{prop:scopeSubtree}

The scope of any subtree $\tau(X,Z) \in \enctr(\spgm(X,Z))$ is $\{X,Z\}$.

\end{proposition} 
\begin{proof} A subtree is obtained with Definition \ref{def:ESPNsubn} by iteratively choosing only one child of each sum node. However, each child of a sum node has the same scope, due to Definition \ref{def:ESPN}, condition $6$. Hence, taking only one child one obtains the same scope as taking all the children. 
\end{proof}

\subsubsection{SPGMs as Mixtures of Subtrees}

\begin{table}[t]
\noindent \centering%
\fbox{\begin{minipage}[c]{12cm}%
\vspace{-10pt}

\begin{align}
P\left(X,Z\right) & =\sum_{\tau\in\enctr(\spgm|Z)}\lambda_{\tau}P_{\tau}(X)\label{pxz}\\
P(X) & =\sum_{\tau\in\enctr(\spgm)}\lambda_{\tau}P_{\tau}(X) \label{mixture-ptau}\\
P_{\tau}(X) & =\prod_{r\in V_{\tau} \colon \vpa(r)=\emptyset }P_{r}(X_{r})\prod_{s\in V_{\tau},t:\in V_{\tau} \colon s \in \vpa(t) }P_{s,t}(X_{t}|X_{s})\label{eq:t}\\
\lambda_{\tau} & =\prod_{s\in O_{\tau}}Q_{s}(k_{s,\tau})\prod_{s\in U_{\tau}}W_{s}(k_{s,\tau})\label{pz}
\end{align}
\end{minipage}} \caption{Probabilistic models related to a SPGM $\spgm(X,Z)$. Symbols $V_{\tau}$, $O_{\tau}$ and $U_{\tau}$ denote respectively the set of Vnodes, Observed Sum Nodes and Unobserved Sum Nodes in a subtree $\tau \in \enctr(\spgm)$.
   $k_{s,\tau}$ denotes
the index of the child of $s$ that is included in $\tau$ (see Definition \ref{def:ESPNsubn}). 
Evaluation of $\spgm$ (Definition \ref{def:SPGM-marg}) is equivalent to inference using the distribution \eqref{pxz}, which is a mixture distribution \eqref{mixture-ptau} of tree graphical models \eqref{eq:t}, whose structure \textit{depends on the context} as specified by the context variables $Z$ of \eqref{pxz}.
}
\label{tab:encoded-probabilities}
\end{table}

In this section, we show that SPGMs can be interpreted as mixtures of trees. Table \ref{tab:encoded-probabilities} lists the notation and probabilistic (sub-)models relevant in this context.

As a first step, we show that inference in a subtree $\tau$ due to Definition \ref{def:ESPNsubn} is equivalent to inference
in a tree graphical model of the form \eqref{eq:treeGm}, multiplied for a constant factor
determined by the sum nodes in the subtree. 

\begin{proposition}\label{prop:infInEncTtr}

Let $\spgm=\spgm(X,Z)$ be a given SPGM, and let $\tau\in\enctr(\spgm)$ be a subtree (Def.~\ref{def:ESPNsubn}) with indicator variables $z_{\tau}$ (Def.~\ref{def:ESPNsubn-3}). Then message passing in $\tau$ is equivalent
to inference using the distribution
\begin{equation}\label{eq:pWithInd}
\lambda_{\tau} P_{\tau}\left(X\right)\prod_{[Z_{s}]_{j}\in z_{\tau}}[Z_{s}]_{j},
\end{equation}
where $P_{\tau}\left( X \right)$ is a tree graphical model of the form \eqref{eq:t}, $\lambda_{\tau} > 0$ is a scalar term obtained by multiplying the weights of all sum nodes in $\tau$ given by \eqref{pz}, and $\prod_{[Z_{s}]_{j}\in z_{\tau}}[Z_{s}]_{j}$ is the product
of all indicator variables in $z_{\tau}$. 
\end{proposition}
\begin{proof} See Appendix \ref{subsec:Proof-of-Proposition-1}. 
\end{proof}

\noindent
The second step consists in noting that $\spgm$ can be written equivalently
as the mixture of all its subtrees.

\begin{proposition}\label{prop:spgmIsMix2}
Evaluating a SPGM $\spgm(X,Z)$ is equivalent to evaluating a SPGM $\spgm'(X,Z)=\sum_{\tau \in\enctr(\spgm)}\tau(X,Z)$.
\end{proposition}
\begin{proof} In Appendix \ref{subsec:Proof-of-Proposition}. 
\end{proof}


\noindent
We are now prepared to state the main result of this section.
\begin{proposition}\label{prop:spgmIsMix2-1}
Let $Y_x\subseteq X,Y_z \subseteq Z$ denote evidence variables with assignment $y_x \in \val(Y_x), y_z\in \val(Y_z)$, respectively, and denote by $x_{\setminus y} \in \Delta(X \setminus Y_{x}), z_{\setminus y} \in \Delta(Z \setminus Y_{z})$ assignments to the remaining variables. 
Evaluating a SPGM $\spgm=\spgm(X,Z)$ with assignment $(y_x,y_z)$ (Definitions \ref{def:Message-passing-in} and \ref{def:SPGM-marg})
is equivalent to performing marginal inference with respect to the distribution \eqref{pxz} as follows:
\begin{equation}\label{eq:main}
P(Y_x=y_x, Y_z=y_z ) = \sum_{\substack{x_{\setminus y} \in \val(X \setminus Y_x) \\ z_{\setminus y} \in \val(Y \setminus Y_z)}} P\big((X \setminus Y_x) = x_{\setminus y}, (Y \setminus Y_z) = z_{\setminus y}, Y_x=y_x, Y_z=y_z\big).
\end{equation}
\end{proposition}

\begin{proof} Due to Propositions \ref{prop:infInEncTtr} and \ref{prop:spgmIsMix2},
the evaluation of $\spgm$ corresponds to performing message passing (Eq. \ref{eq:treeMsgPass}) with the mixture distribution
\begin{equation*}
\sum_{\tau\in\enctr(\spgm)}\bigg(\prod_{[Z_{s}]_{j}\in z_{\tau}}[Z_{s}]_{j}\bigg)\lambda_{\tau}P_{\tau}\left(X\right).
\end{equation*}
We now note that the term $\left(\prod_{[Z_{s}]_{j}\in z_{\tau}}[Z_{s}]_{j}\right)$
attains the value $1$ only for the subset of trees compatible with the assignment $Y_z=y_z$ and
$0$ otherwise (since some indicator in the product is $0$), that is for subtrees in the set $\enctr(\spgm|Y_z=y_z)$ (Definition \ref{def:ESPNsubn-2}).
Therefore, the sum can be rewritten as $\sum_{\tau\in\enctr(\spgm|Z)}$, 
and the indicator variables (with value $1$) can be removed, which results in \eqref{pxz}. 
The proof is concluded noting that computing message passing in a mixture of trees with assignment $y_x , y_z$ corresponds to computing marginals $P(Y_x=y_x, Y_z=y_z )$ in the corresponding distribution (Section \ref{subsec:Directed-Graphical-Models}), hence Eq. \eqref{eq:main} follows. 
\end{proof}
\begin{example}
All subtrees of the SPGM $\spgm(X,Z)$ shown by Fig.~\ref{fig:spgm1} are shown as tree graphical models by Fig.~\ref{fig:spgm3}. The probabilistic model encoded by the SPGM is a mixture of these subtrees whose structure depends on the context variables $Z$.
\end{example}
The propositions above entail the crucial result that the probabilistic model of a SPGM  is a mixture of trees where the mixture size grows
\emph{exponentially} with the SPGM size (Definition \ref{def:ESPNsubn}),
but in which the inference cost grows only \emph{polynomially} (Proposition
\ref{proposition:linearCost}). Hence, \emph{very large} mixtures models
can then be modelled \textit{tractably}. This property is obtained by modelling
trees $\tau \in \enctr(\spgm|Z)$ by combining sets of \emph{shared subtrees}, selected through context variables $Z$, and by computing inference
in shared parts only once (cf.~the example in Fig. \ref{fig:ESPNexample}).

\subsubsection{Conditional and Contextual Independence}

In this section we discuss conditional and contextual independence
semantics in SPGMs, based on their interpretation as mixture model. 

\begin{definition}[Context-dependent paths]
Consider variables $A\in X,B\in X$ and a context  $z\in \val(\overline{Z})$ with $\overline{Z}\subseteq Z$.
\begin{itemize}
\item The set $\pi(A,B)$ is the set of all directed paths in $\spgm$ going from a
Vnode with label $A$ to a Vnode with label $B$. 
\item The set $\pi(A,B|\overline{Z}=z)\subseteq\pi(A,B)$ is the subset
of paths in $\pi(A,B)$ in which all the indicator variables over $\overline{Z}$
(Definition \ref{def:ESPNsubn-3}) are in state $z$. 
\end{itemize}
\end{definition}

\begin{proposition}[\emph{D-separation} in SPGMs]\label{prop:indepSem} Consider a SPGM $\spgm(X,Z)$, variables $A, B, C \in X$ and a context $z\in \val(\overline{Z})$ with $\overline{Z} \subseteq Z$. The following properties hold for the probabilistic model $\spgm$ corresponding to Eq.~\eqref{pxz}: 
\begin{enumerate}
\item $A$ and $B$ are \emph{independent} iff $\pi(A,B)=\emptyset$ and
$\pi(B,A)=\emptyset$ (there is no directed path from $A$ to $B$). 
\item $A$ and $B$ are \emph{conditionally independent} given $C$ if all
directed paths $\pi(A,B)$ and $\pi(B,A)$ contain $C$. 
\item $A$ and $B$ are \emph{contextually independent} given context $\overline{Z}=z$
iff $\pi(A,B|\overline{Z}=z)=\emptyset$ and $\pi(B,A|\overline{Z}=z)=\emptyset$. 
\item $A$ and $B$ are \emph{contextually and} \emph{conditionally independent}
given $C$ and context $\overline{Z}=z$ iff all paths $\pi(A,B|\overline{Z}=z)$
and $\pi(B,A|\overline{Z}=z)$ contain $C$. 
\end{enumerate}
\end{proposition}

\begin{example}
In Fig. \ref{fig:ESPNexample}, $A$ and $D$ are
conditionally independent given $C$; $A$ and $C$ are conditionally
independent given $B$ and context $Z_{1}=1$. 
\end{example}
\begin{proof}
In a mixture of trees, conditional independence of $A,B$
given $C$ holds \emph{iff} for every tree in the mixture the path
between $A$ and $B$ contains $C$ (D-separation, see \citep{Cowell-et-al-03}).
If D-separation holds for all paths in $\pi(B,A|z)$ then it holds
for all the subtrees compatible with $Z=z$. But $P(X,Z)$ is the
mixture of all subtrees compatible with assignment $z$ (Eq. \eqref{pxz}).
Hence the result follows.
\end{proof}

The proposition above provides SPGMs with a high level representation of both contextual and conditional independence. 
This is obtained by using different variables sets $X$ and $Z$ for the two different roles. The set $X$ appears in Vnodes and entails conditional independences due to D-separation, with close similarity to tree graphical models (Proposition \ref{prop:indepSem}). The set $Z$ enable contextual independence through the selection of tree branches via sum nodes. 

Note also that using the set of paths $\pi$ allows to infer conditional
and contextual independences without need to check all the individual
subtrees, whose number can be exponentially larger than the cardinality
of $\pi$.

\subsubsection{Related Models}\label{subsec:Discussion}

SPGMs are closely related to hierarchical mixtures of trees (HMT) \citep{Jordan94hierarchicalmixtures} and \emph{generalize} them. Like SPGMs, HMTs allow a compact representation of mixtures of trees by using a hierarchy of "choice nodes"  where different trees are selected at each branch (as sum nodes do in SPGMs). While in HMTs the choice nodes separate the graph into disjoint branches and thus an overall tree structure is induced, however, SPGMs enable to use a DAG structure where parts of the graph towards the leaves can appear in children of multiple sum nodes. 

The probabilistic model encoded by SPGMs also has a close connection to
Gates \citep{GatesMinka2009} and the similarly structured Mixed Markov Models (MMM) \citep{MixedMarkovModels}. Gates enable contextual independence in graphical models by including the possibility of activating/deactivating factors based on the state of some context variables. Regarding SPGMs, the inclusion/exclusion of factors in subtrees $\enctr(\spgm|Z)$ depending on values of $Z$ can be seen as a gating unit that enables a full set of tree factors to be active, which suggests to identify a SPGM as a Gates model. 
On the other hand, SPGMs restrict inference to a family of models in which inference is tractable by construction, while inference in Gates generally is intractable. 
In addition, SPGMs allow an interpretation as mixture models, which is not the case for Gates and MMMs.  

Finally, we remark that SPGM subtrees closely
parallel the concept of SPN subnetworks, first described in \citet{SPNdiscrimLearning2012}
and then formalized in \citet{zhao16}. However, while subnetworks
in SPNs represent simple factorizations of the leaf distributions (which
can be represented as graphical models without edges), subtrees in
SPGMs represent tree graphical models (which include edges). 

\begin{figure}[tb]
\center{
\hfill{}
\subfloat[\label{fig:msgsumObs}Observed sum node (Eq. \eqref{eq:msgSumNode}).]{\includegraphics[height=2.2cm]{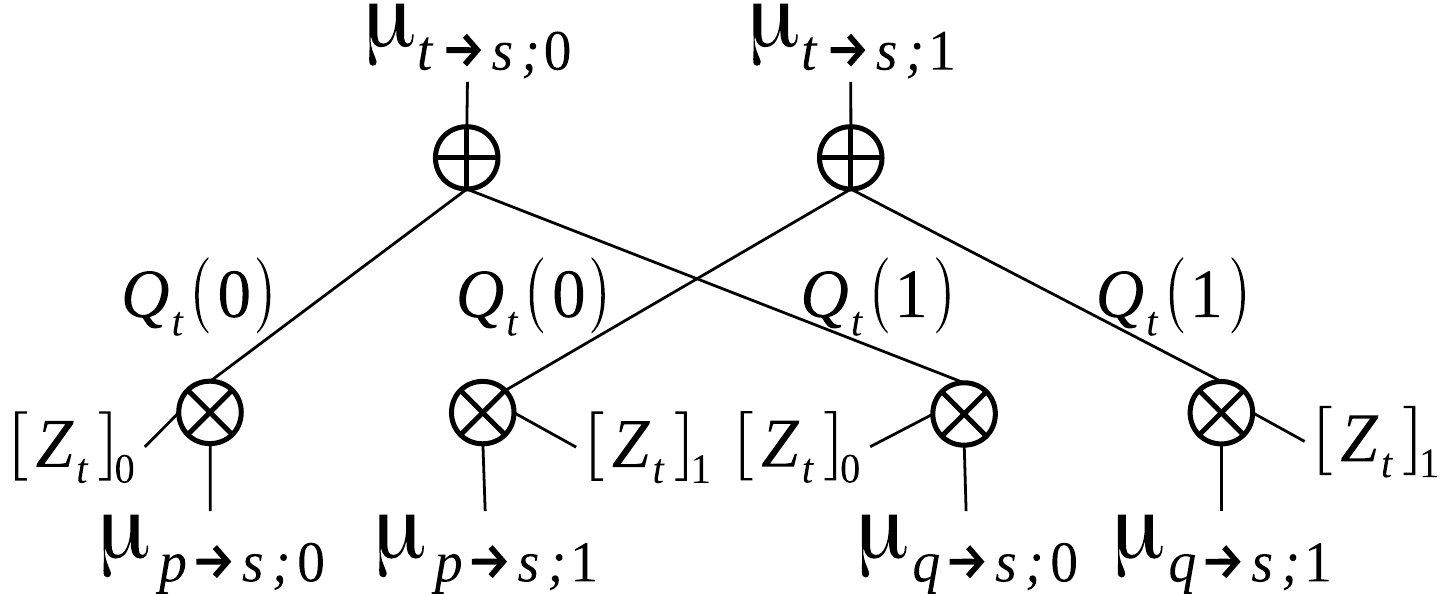}}
\hfill{}
\subfloat[\label{fig:msgsumUnobs}Unobserved sum node (Eq. \eqref{eq:msgSumUnobs}).]{\includegraphics[height=2.2cm]{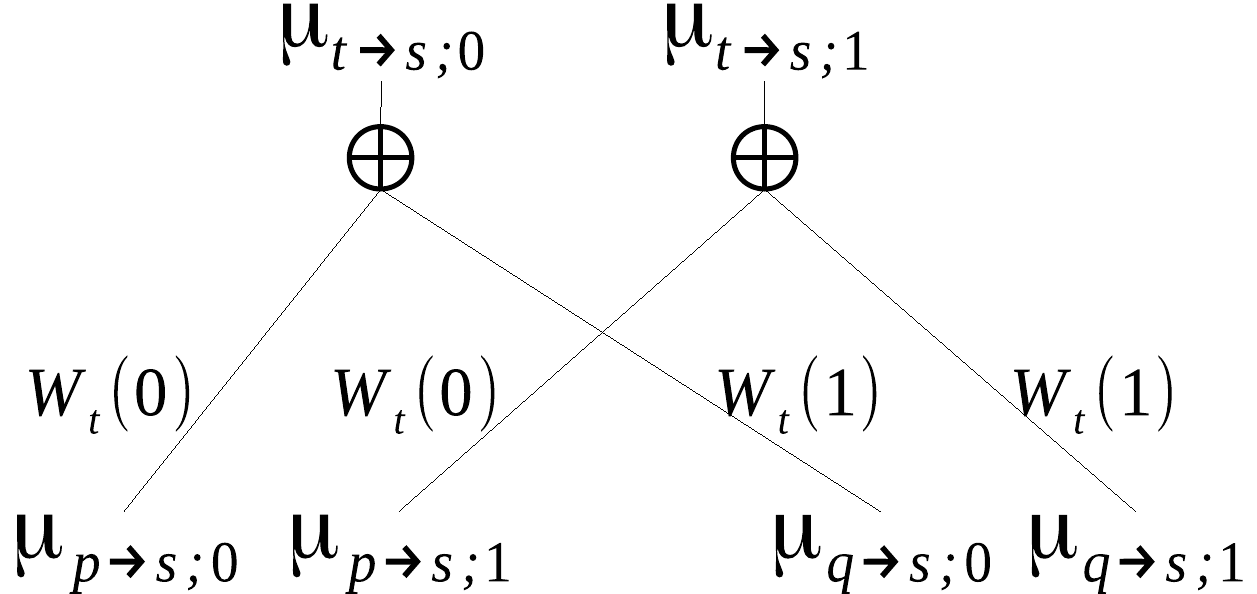}}
\hfill{}\hfill{}
\\
\hfill{}
\subfloat[\label{fig:msgprod}Product node (Eq. \eqref{eq:msgProd}).]{\includegraphics[height=2.2cm]{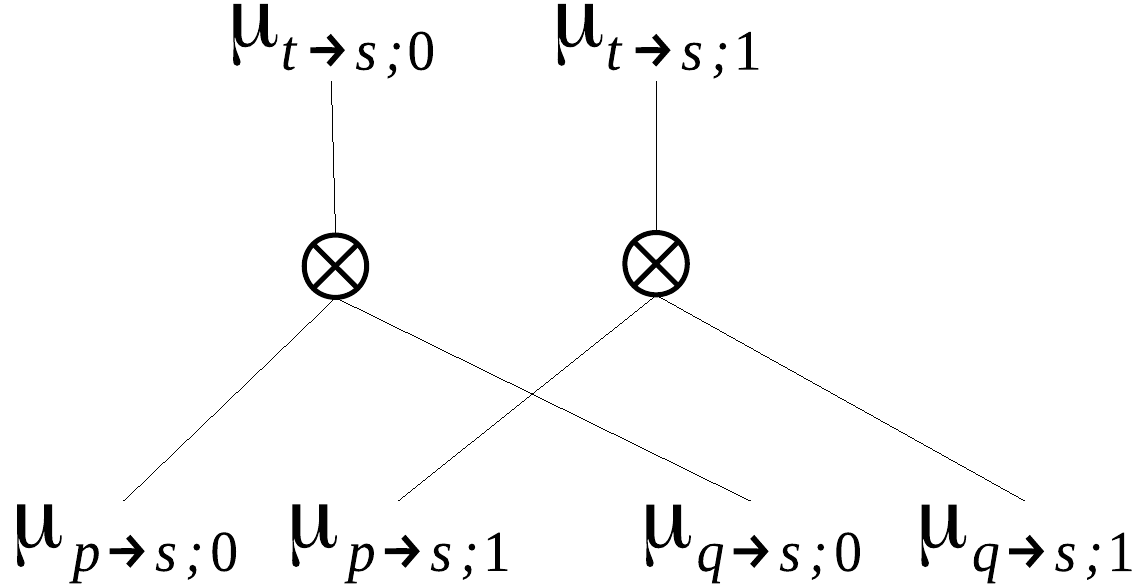}}
\hfill{}
\subfloat[ \label{fig:msgv}Vnode (Eq. \eqref{eq:msgPnode}).]{\includegraphics[height=3.7cm]{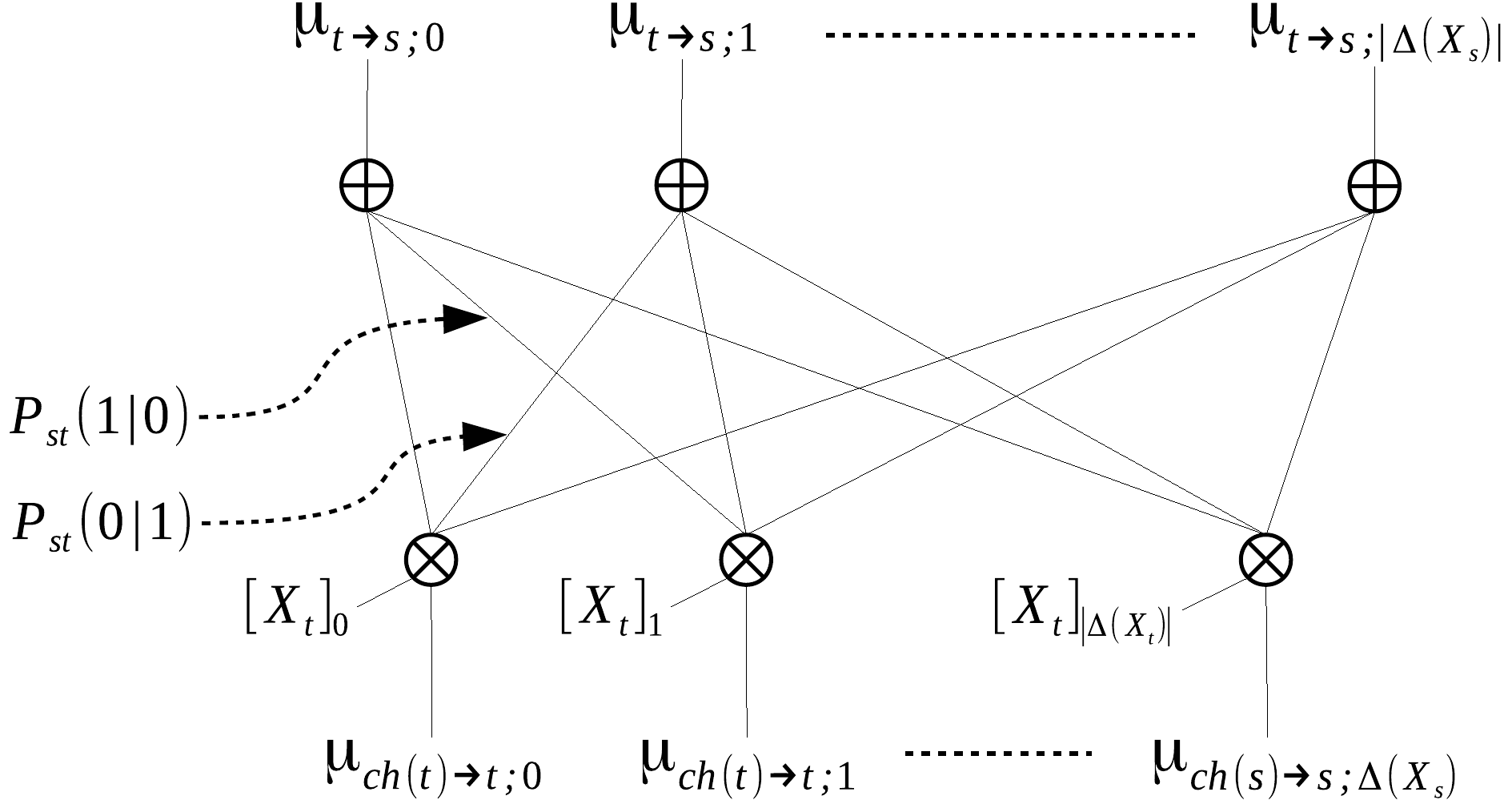}}
\hfill{}\hfill{}
}

\caption{Representation of message passing equations as SPNs. For better visibility, sum and product nodes are assumed to have only two children $p,q$ and with binary variables. \label{fig:Representation-of-messages}}
\end{figure}

\subsection{Interpretation as SPN \label{sec:asSpn}}

In this section we discuss SPGMs as a high level, fully general representation
of SPNs as defined by Definition \ref{def:SPN}.

\subsubsection{SPGMs encode SPNs}

\begin{proposition}\label{prop:spgmAsSpn}

The message passing procedure in $\spgm\left(X,Z\right)$ encodes
a SPN $S(X,Z)$.

\end{proposition}

\begin{proof} In Appendix \ref{proof:asSpn}. \end{proof}

Note that in the encoded SPN, each SPGM message is represented by
a set of sum nodes, which can be seen immediately from Fig.~\ref{fig:Representation-of-messages}.
Each sum node in the set represents the value of the message corresponding
to a certain state of the output variable (namely, $\mu_{t\rightarrow s;j}$ for each $j$). This entails an increase in representation size (but not in inference cost) by a $|\val_{max}|^{2}$ factor. Note also that the role
of nodes for implementing conditional and contextual independence is lost during
the conversion to SPN, since both SPGM Vnode and SPGM sum node messages translate into a set of SPN sum nodes. 

\begin{proposition}\label{prop:sameExpress}

SPGMs are as \emph{expressive} as SPNs, in the sense that if a distribution $P(X,Z)$
can be represented as a SPN with inference cost $C$, then it can
also be represented as a SPGM with inference cost $C$ and vice versa.

\end{proposition}

\noindent
\emph{Proof Sketch.} Firstly, due to Proposition \ref{prop:spgmAsSpn}, 
SPNs are at least as expressive as SPGMs since they encode a SPN via message passing. Secondly, any SPN $S$ can
be transformed into an equivalent SPGM $\spgm$ by simply replacing
the indicator variable $[A]_{a}$ in SPN leaves with Vnodes $s$ associated
to variable $A$ and unary probability $P_{s}(A)=[A]_{a}$ (notice that pairwise probabilities do not appear). It is immediate to see that by 
doing so all the conditions of Definition \ref{def:ESPN} are satisfied,
and evaluating $\spgm$ with message passing yields $S$. As a consequence,
SPGMs are at least as expressive as SPNs. \qed

\subsubsection{Discussion}

Propositions \ref{prop:spgmAsSpn} and \ref{prop:sameExpress}, together with 
the connections to graphical models worked out above, enable an interpretation of a SPGM $\spgm$ as
a high-level representation of the encoded SPN $S$. These generalizes what the introductory example demonstrated by comparing Fig.~\ref{fig:spnVsGm-3} with Fig.~\ref{fig:spnVsGm-4}. 

The SPGM representation is more \emph{compact} than SPNs because employing variable nodes as in graphical models enables 
to represent conditional independences through message passing. Passing
from an SPGM to the SPN representation entails an increase in the model
size due to the expansion of messages
by a $N_{pa}|\val_{max}|^{2}$ factor (see Definition \ref{def:Message-passing-in}
and Fig. \ref{fig:Representation-of-messages}).\footnote{Note, however, that inference cost remains identical in $\spgm$ and $S$.} 

The SPGM representation allows a high level representation of \emph{conditional} independences through Vnodes and D-separation, and of \emph{contextual} independences through the composition of subtrees due to context variables $Z$.
In contrast, the roles of contextual and conditional independence in SPNs is hard to decipher (Fig. \ref{fig:spnVsGm-3}) because there is no distinction between nodes created by messages sent from SPGM sum nodes (implementing contextual independence) and messages generated from SPGM Vnodes (implementing conditional independence), both of which are represented as a set of SPN sum and product nodes (see Fig. \ref{fig:Representation-of-messages}). In addition, there is no distinction between contextual variables $Z$ and Vnode variables $X$. 

Note that SPGMs are not more compact than SPNs in situations in which there are no conditional independences that can be expressed by Vnodes. However, we postulate that the co-occurrence of conditional and contextual independences creates relevant application scenarios (as shown in Section \ref{sec:Structure-Learning}) and enables connections between SPNs and graphical models that can be exploited in future work.

Finally, the interpretation of SPGMs as SPN also allows to translate all methods and procedures available for SPNs to SPGMs. 
These include jointly computing the \emph{marginals} of all variables by derivation \citep{Darwiche2003}, 
with  time and memory linear cost in the number of edges in the SPN. 
In addition, Maximum a Posteriori queries can be computed simply by substituting the
sums in Eqs.~\eqref{eq:messages} by $\max$ operations. We leave the exploration of these aspects
to future work, since they are not central for our present discussion.


\section{Learning SPGMs\label{sec:Structure-Learning}}

In this section, we exploit the relations between graphical models and SPNs embodied by SPGMs and present an algorithm for \textit{learning the structure} of SPGMs.

\subsection{Preliminaries\label{subsec:Structure-Learning-In}}

\emph{Structure learning} denotes the problem of learning both the parameters of a probability distribution $P(X|\mathcal{G})$ and the structure of the underlying graph $\mathcal{G}$.
As both the GM and the SPN represented by a given SPGM due to Sections \ref{sec:asMixTrees} and \ref{sec:asSpn} involve the same graph, the problem is well defined from both viewpoints.

Let $X=\{X_j\}_{j=1}^M$ be a set of $M$ discrete variables. 
Consider a training set of $N$ i.i.d samples  $D = \left\{ x^{i}\right\}_{i=1}^{N} \subset \val(X)$, used for learning. 
Formally, we aim to find the graph $\mathcal{G}^*$ governing the distribution $P(X)$ which maximizes the log-likelihood
\begin{equation}\label{eq:LL}
\mathcal{G}^{*}(X)=\arg\max_{\mathcal{G}}\text{LL}(G)=\arg\max_{\mathcal{G}}\sum_{i=1}^{N}\ln P(x^{i}|\mathcal{G})
\end{equation}
or the weighted log-likelihood 
\begin{equation}\label{eq:wLL}
\mathcal{G}^{*}(X)=\arg\max_{\mathcal{G}} \text{WLL}(G,w)=\arg\max_{\mathcal{G}}\sum_{i=1}^{N} w_i \ln P(x^{i}|\mathcal{G}),\qquad w_{i} \geq 0,\quad i=1,2,\dotsc,N.
\end{equation}

\paragraph{Learning Tree GMs.}
Learning the structure of GMs generally is NP-hard. For discrete tree GMs however the maximum likelihood solution $T^{*}$
can be found with cost $O(M^{2}N)$ using the Chow-Liu
algorithm (\citet{ApproxDistrWithTrees}).

Let $X_s, X_t \in X$ be discrete random variables ranging of  assignments in the sets $\val(X_s), \val(X_t)$. Let $N_{s;j}$
and $N_{st;jk}$ respectively count the number of times $X_s$ appears
in state $j$ and $X_s,X_t$ appear jointly in state $j,k$ in the training set $D$. 
Finally, define empirical probabilities $\overline{P}_{s}(j)=N_{s;j}/N$ and
$\overline{P}_{st}(k|j)=N_{st;jk}/N_{s;j}$. The Chow-Liu algorithm comprises the following steps:
\begin{enumerate}
\item Compute the \emph{mutual information} $\mathbb{I}_{st}$ between all variable pairs $X_s,X_t$,
\begin{equation}\label{eq:empirical-Ist}
\mathbb{I}_{st}=\sum_{j\in\val(X_{s})}\sum_{k\in\val(X_{t})}\overline{P}_{s,t}(j,k)\ln\frac{\overline{P}_{s,t}(j,k)}{\overline{P}_{s}(j)\overline{P}_{t}(k)}.
\end{equation}

\item Create an undirected graph $\overline{\mathcal{G}}=(\overline{\mathcal{V}},\overline{\mathcal{E}})$ with adjacency matrix $\mathbb{I}=\{\mathbb{I}_{st}\}_{s,t \in \mathcal{V}}$ and compute the corresponding Maximum Spanning Tree $\overline{T}$.

\item  Obtain the directed tree $T^{*}$ by choosing an arbitrary
node of $\overline{T}$ as root and using empirical probabilities $\overline{P}_{s}(X_{s})$
and $\overline{P}_{st}\left(X_{t}|X_{s}\right)$ in place of corresponding terms in Eq.~\eqref{eq:P-directed-GM}. 
\end{enumerate}

If the weighted log-likelihood \eqref{eq:wLL} is used as objective function, the algorithm remains the same. The only difference concerns the use of \textit{weighted relative frequencies} for defining the empirical probabilities of \eqref{eq:empirical-Ist}: $\hat{N}_{s,j}=\frac{1}{\hat{N}_{w}}\sum_{i=1}^{N}\delta(x_{s}^{i}=j) w_{i}$
and $\hat{N}_{st,jk}=\frac{1}{\hat{N}_{w}}\sum_{i=1}^{N}\delta(x_{s}^{i}=j,x_{t}^{i}=k) w_{i}$,
where $\hat{N}_{w}=\sum_{i=1}^{N}w_{i}$ and $x_{s}^{i}$ denotes the state of variable $X_s$ in sample $x^i$. 


\paragraph{Learning Mixtures of Trees.}
We consider mixture models of the form $P(X)=\sum_{k=1}^{K}\lambda_{k} P_k(X|\theta^k)$
with tree GMs $P_{k}\left(X \right | \theta^k),\; k=1,\dotsc,K$,  corresponding parameters $\{\theta^k\}_{k=1}^{K}$ and non-negative mixture coefficients $\{\lambda_{k}\}_{k=1}^K$ satisfying $\sum_{k=1}^{K}\lambda_{k}=1$.
While inference with mixture models is tractable as long as it is tractable with its individual mixture components, maximum likelihood generally is NP hard. A local optimum can be found with Expectation-Maximization (EM) \citep{EMorigDempster}, whose pseudocode is shown in Algorithm
\ref{MT_EM}. The M-step (line $8$) involves the weighted
maximum likelihood problem and determines $\theta^{k}$ using the Chow-Liu algorithm described above. 

It is well known that each EM iteration does not decrease the log-likelihood, hence it approaches a local optimum. 


\begin{algorithm}[tb]
\caption{EM for Mixture Models(${P_{k}},{\lambda_{k}},D$)} 
\label{MT_EM}
\begin{algorithmic}[1] 
\Input Initial model $P(X|\theta )=\sum_k{\lambda_{k} P_{k}(X)}$, training set $D$
\Output  ${P_{k}(X)}$, $\lambda_{k}$ locally maximizing $\sum_{i=1}^{N}\ln P(x^{i}|\theta)$
\Repeat 

	\ForAll{ $k\in{1...K},i\in{1...N}$} // E-step 
		\State $\gamma_{k}(i)\leftarrow\frac{\lambda_{k} P_{k} (x_{i})}{\sum_{k'}\lambda_{k'} P_{k'}(x_{i})}$ 
		\State $\Gamma_{k} \leftarrow \sum_{i=1}^{N} {\gamma_{k}(i)} $ 
		\State $w_{i}^{k}\leftarrow \gamma_{k}(i) / \Gamma_{k}$ 
	\EndFor 


	\ForAll {$k\in{1...K}$ } // M-step 
		\State $\lambda_{k}\leftarrow\Gamma_{k}/N$ 
		\State $\theta^{k}\leftarrow\arg\max_{\theta^{k}}\sum_{i=1}^{N}w_{i}^{k}\ln P_{k}(x^{i}|\theta^{k})$\label{eq:mstep} 
	\EndFor

\Until{convergence} 

\end{algorithmic}
\end{algorithm}

\paragraph{Learning SPNs.}
Let $S(X)$ denote a SPN, $\mathcal{G}=(\mathcal{V},\mathcal{E})$ and graph with edge weights. Both \textit{structure learning} (optimizing $\mathcal{G}$ and $W$) and \textit{parameter learning} (optimizing $W$ only) are NP-hard in SPNs \citep{ArithmCircuitsAndNetworkPoly2}. Hence, only algorithms that seek a local optimum can be devised.
\begin{description}
\item{\textbf{Parameter learning}} can be performed by directly applying the EM iteration for mixture models, while efficiently exploiting the interpretation of SPNs as a large mixture model with shared parts \citep{DesanaS16}. 

To describe EM for SPNs, which will be used in a later section, we need some additional notation. Consider a node $q\in \mathcal{V}$, and let $S_q$ denote the sub-SPN having node $q$ as root.  If $q$ is a Sum Node, then by Definition \ref{def:SPN} a weight $w_j^q$ is associated to each edge $(q,j)\in \mathcal{E}$.  
Note that evaluating $S(X=x)$ entails computing $S_q(X=x)$ for each node $q\in \mathcal{V}$ due to the recursive structure of SPNs. Hence $S(x)$ is function of $S_q(x)$. The derivative  ${\partial S\left(x\right)}/{\partial S_{q}(x)}$ can be computed with a root-to-leaves pass requiring $O(|\mathcal{E}|)$ operations \citep{SPN2011}. 

With this notation, the EM algorithm for SPNs iterates the following steps:
\begin{enumerate}
\item \textit{E step}. Compute 
for each Sum Node $q\in \mathcal{V}$ and each $j \in ch(q)$  
\begin{equation}
\beta_{j}^{q}=w_{j}^{q}\sum_{n=1}^{N}S\left(x^{n}\right)^{-1}\frac{\partial S\left(x^{n}\right)}{\partial S_{q}}S_{j}\left(x_{n}\right).
\end{equation}

\item \textit{M step}. Update weights for each Sum Node $q\in \mathcal{V}$ and each $j \in ch(q)$ by $w_{j}^{q} \leftarrow \beta_{j}^{q}/\sum_{(q,i)\in \mathcal{E}}\beta_{i}^{q}$, where $\leftarrow$ denotes assignment of a variable. 

\end{enumerate}
Since all the required quantities can be computed in $O(|\mathcal{E}|)$ operations, EM has a cost $O(|\mathcal{E}|)$ per iteration (the same as an SPN evaluation). 

In some SPN applications, weights are shared among different edges (see e.g. \citet{SPNdiscrimLearning2012}, \citet{SPNlangMod2014} and \citet{amer2015spnActivityRec}). Then the procedure still maximizes the likelihood locally. Let $\overline{V} \subseteq V$ be a subset of Sum Nodes with shared weights, in the sense that the set of weights $\{w^q_j\}_{j\in ch(q)}$ associated to incident edges $(q,j)\in \mathcal{E}$ is the same for each node $q \in \overline{V}$. The EM update of a shared weight $w^q_j$ reads (cf.~ \citet{DesanaS16})
\begin{equation}\label{eq:EMparams}
w^q_j \leftarrow \frac{\sum_{q\in \overline{V}}\beta_{j}^{q}}{\sum_{i}\sum_{q\in \overline{V}}\beta_{i}^{q}},\qquad (q,j) \in \mathcal{E}.
\end{equation}

\item{\textbf{Structure learning}}
can be more conveniently done with SPNs than with graphical models, because tractability of inference is always guaranteed and hence not a limiting factor for learning the model's structure. Several greedy algorithms for structure learning were devised (see Section \ref{sec:Empirical-Evaluation}), which established SPNs as state of the art models for the estimation of probability distributions. 
We point out that most approaches employ a recursive procedure in which children of sum and Product Nodes are generated on \emph{disjoint} subsets of the dataset, thus obtaining a \textit{tree SPN}, while SPNs can be more generally defined on DAGs. Recently, \citep{Rahman16mergeSPN} discussed the limitations of using tree structured SPNs as opposed to DAGs, and addressed the problem of post-processing SPNs obtained with previous methods, by merging similar branches so as to obtain a DAG. 

Our method proposed in Section \ref{sec:Structure-Learning-alg} is the first one that directly estimates a DAG-structured SPN. 
\end{description}

\subsection{Parameter Learning in SPGMs  }

Parameters learning in a SPGM $\spgm$ can be done by interpreting $\spgm$ as a SPN encoded by message passing
(Proposition \ref{prop:sameExpress}) and directly using any available SPN parameter learning method (these include EM seen in Section \ref{subsec:Structure-Learning-In} and others \citep{SPNdiscrimLearning2012}). Hence, we do not discuss this aspect further. 

Note however that Sum Node messages (Definition \ref{def:Message-passing-in}) require
weight sharing between the SPN Sum Nodes. EM for SPNs with weight tying
is addressed in Section \ref{subsec:Structure-Learning-In}. 

\subsection{Structure Learning in SPGMs  }\label{sec:Structure-Learning-alg}

\begin{figure}[tb]
\quad{}\quad{}\quad{}\quad{}\quad{}\quad{}
\subfloat[\label{fig:msgprod-1}]{ \includegraphics[height=5cm]{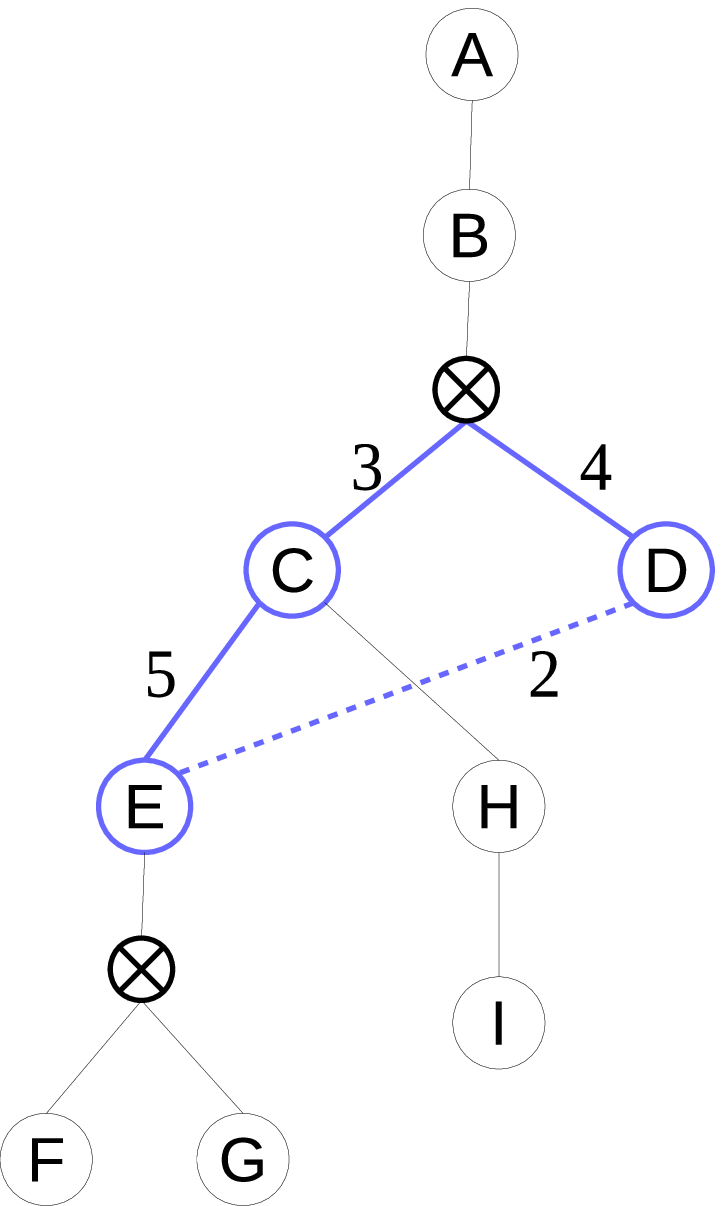} }\hfill{}
\subfloat[\label{fig:msgprod-2}]{ \includegraphics[height=5cm]{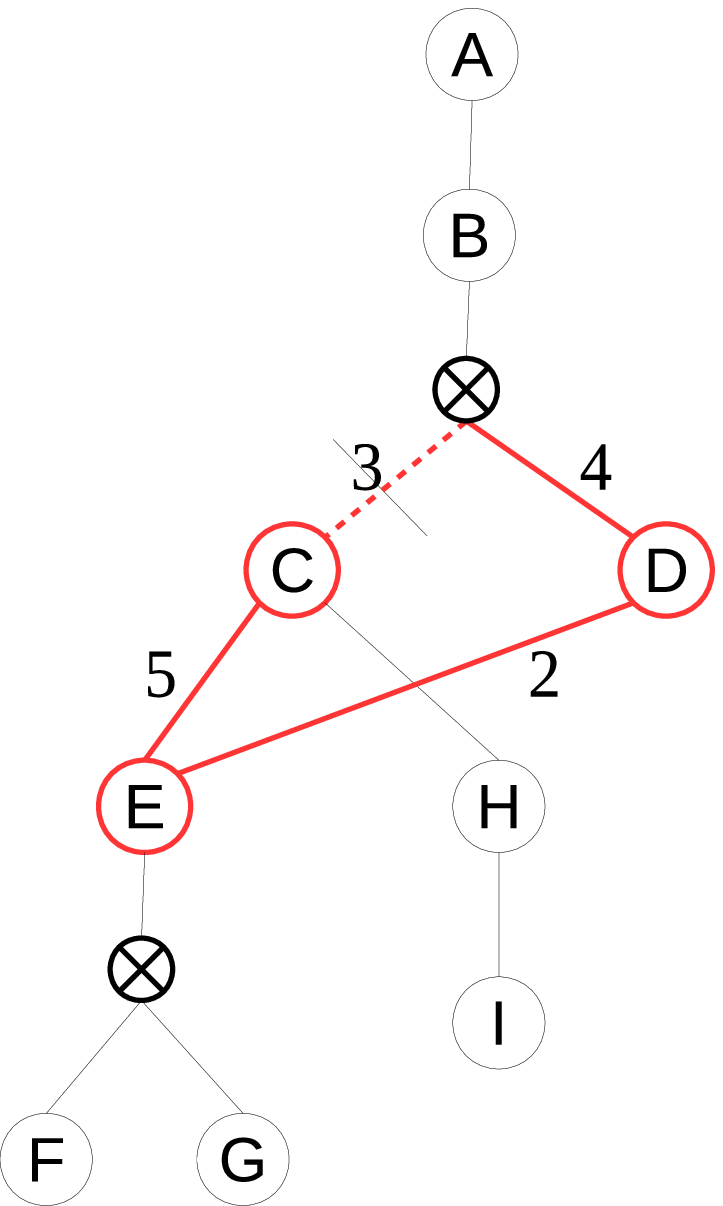} }\hfill{}
\subfloat[\label{fig:msgprod-3}]{ \includegraphics[height=5cm]{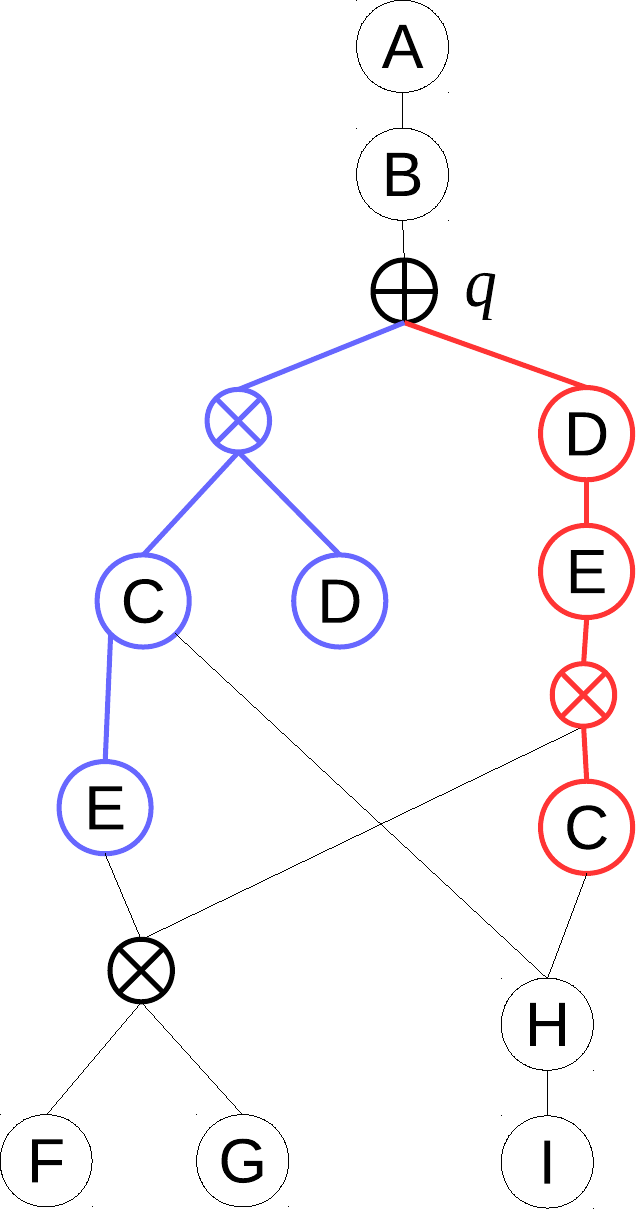} }
\quad{}\quad{}\quad{}\quad{}\quad{}\quad{}\caption{Graphical representation of edge insertion. For simplicity, we start
with a SPGM representing a single MST. Fig. \ref{fig:msgprod-1}:
Inserting $(D,E)$ creates a cycle (blue). Fig. \ref{fig:msgprod-2}:
Removing the minimum edge in the cycle (except $(D,E)$) gives the
MST containing $(D,E)$. Fig. \ref{fig:msgprod-3}: The red MST is
inserted into $\spgm$ sharing the common parts. \label{fig:mixMaxSteps34}}
\end{figure}

Structure learning is an important aspect of tractable inference models, 
thus it is crucial to provide a structure learning for SPGMs. Furthermore, 
it is useful to provide a first example of how the new \emph{connections} between GMs and SPNs can be exploited in practice.

We propose a structure learning algorithm based on the Chow-Liu algorithm
for trees (Section \ref{subsec:Directed-Graphical-Models}). We start
observing that edges with large Mutual Information can be excluded
from the Chow-Liu tree, thus losing relevant correlations between variables
(Fig.  \ref{fig:learnEx-2}, left). An approach to address this problem, inspired by  \citep{LearningWithMixturesOfTrees}, is
to use a mixture of spanning trees such that the $k$-best edges are
included in at least one tree. We anticipate that the trees obtained in this way \emph{share a large
part} of their structure (Fig.  \ref{fig:learnEx-3}), hence the mixture
can be implemented efficiently as a SPGM.

\paragraph{Algorithm Description.}

We describe next \textit{LearnSPGM}, a procedure to learn structure and parameters of a SPGMs which locally maximizes the weighted log-likelihood \eqref{eq:wLL}. We use the notation of Section \ref{subsec:Structure-Learning-In}. 

The algorithm learns a SPGM $\spgm$ in three main steps (pseudocode in Algorithm \ref{alg:LearnSPGM}). First, $\spgm$ is initialized to encode the Chow-Liu tree $T^{*}$ (Fig. \ref{fig:msgprod-1}) -- that is, $\enctr(\spgm)$ includes a single subtree (Definition \ref{def:ESPNsubn}) $\tau*$ corresponding to $T^{*}$. Then, we order each edge $(s,t)\in \overline{\mathcal{E}}$ which was not included in $T^{*}$ by decreasing mutual information $\mathbb{I}_{st}$, collecting them in the ordered set $Q$. Finally, we insert each edge $(s,t)\in Q$ in $\spgm$ with the sub-procedure \textit{InsertEdge} described below, until log-likelihood convergence or a given maximum size of $\spgm$ is reached.

\textit{InsertEdge($\spgm,T^{*},(s,t)$)} comprises three steps:
\begin{enumerate}

\item Compute the maximum spanning tree over $\overline{\mathcal{G}}$ which includes $(s,t)$, denoted as $T_{st}$. 
Finding $T_{st}$ can be done efficiently by first inserting edge $(s,t)$
in $T^{*}$, which creates a cycle $\mathcal{C}$ (Fig.  \ref{fig:msgprod-1}),
then removing the minimum edge in $\mathcal{C}$\emph{ except} $(s,t)$
(Fig.  \ref{fig:msgprod-2}). The potentials in $T_{st}$ are set as empirical
probabilities $\overline{P}_{st}$ according to the Chow-Liu algorithm. 
Notice that trees $T^{*}$ and $T_{st}$ have identical
structure up to $\mathcal{C}$ and can then be written as $T^{*}=T^{1}\cup\mathcal{C}^{'}\cup T^{2}$
and $T_{st}=T^{1}\cup\mathcal{C}^{''}\cup T^{2}$, where $\mathcal{C}^{'}=T^{*}\cap\mathcal{C},\mathcal{C}^{''}=T_{st}\cap\mathcal{C}$. 

\item Add $T_{st}$ to the set $\enctr(\spgm)$ by sharing the structure in common with $T^{*}$ ($T^{1}$ and $T^{2}$ above). 
To do this, first identify the edge $(s,t)$ s.t. $s \in T^{1}$ and $t \in \mathcal{C}^{'}$ (e.g. $(B,C)$ in Fig. \ref{fig:msgprod-2}). Then, create a non-Observed Sum Node $q$, placing it between $s$ and $t$, unless such node is already present due to previous iterations (see $q$ in Fig. \ref{fig:msgprod-3}). 

At this point, one of the child branches of $q$ contains $\mathcal{C}^{'}\cup T^{2}$. We now add a new child branch containing $\mathcal{C}^{''}$ (Fig. \ref{fig:msgprod-3}). Finally, we connect nodes in $\mathcal{C}^{''}$ to their descendants in the shared section $ T^{2}$. 
The insertion maintains $\spgm$ valid since the $X$ and $Z$-scope of any
node in $\spgm$ does not change. Furthermore, inserting $\mathcal{C}^{''}$ in this way we add a subtree representing  $T_{st}$ in $\enctr(\spgm)$, selected by choosing the child of $s$ corresponding to $\mathcal{C}^{''}$. 

\item Update the weights of incoming edges of the Sum Node $q$ by using Eq. \ref{eq:EMparams} on the set of SPN nodes generated by $q$ with Eq.\eqref{eq:msgSumNode} during the conversion to SPN (see Proposition \ref{prop:spgmAsSpn}).\label{updateWstep}  

\end{enumerate}

\paragraph{Convergence}
It is possible to prove that the log-likelihood \emph{does not decrease} at each insertion, and thus the initial Chow-Liu tree provides a \emph{lower bound} for the log-likelihood. 

\begin{proposition} Each application of \textit{InsertEdge} does not decrease the log-likelihood of the SPGM (Eq. \eqref{eq:wLL}). \label{prop:convergence}
\end{proposition}
\emph{Proof.} \textit{InsertEdge} adds the branch $\mathcal{C}^{'}\cup T^{2}$ to Sum Node $q$, hence it adds a new incident edge and a corresponding weight to $q$.  
We now note that computing weight values using Eq. \eqref{eq:EMparams} (step $3$  in \textit{InsertEdge} above) allows to find the optimal weights of edges incoming to $q$ considering the other edges fixed, as shown e.g. in \citet{DesanaS16}.\footnote{However,
this is not the globally optimal solution when the other parameters are free.} 
Since the new locally optimal solution includes the weight configuration of the previous iteration, which is simply obtained by setting the new edge weight to $0$ and keeping the remaining weights, the log-likelihood does not increase at each iteration.

\begin{proposition} The Chow Liu tree log-likelihood is a lower bound for the log-likelihood of a SPGM obtained with \textit{LearnSPGM}. 
\end{proposition}
\emph{Proof.} Follows immediately from Proposition \ref{prop:convergence}, noting that the SPGM is initialized as the Chow Liu tree $T^{*}$.

\paragraph{Complexity.}
Steps $1$ and $2$ of \textit{InsertEdge} are inexpensive, only requiring a number of operations linear in the number of edges of the Chow Liu tree $T^*$. The per iteration complexity of \textit{InsertEdge} is dominated by step $3$, in which
the computation of weights through Eq. \ref{eq:EMparams} requires evaluating the SPGM for the whole dataset. 
Although evaluation of a SPGM is efficient (see Proposition \ref{proposition:linearCost}), this can still be too costly for large datasets. We found empirically that assigning  weights proportionally to the mutual information of the inserted edge provides a reliable empirical alternative, which we use in experiments.

\subsection{Learning Mixtures of SPGMs\label{subsec:Learning-Mixtures-of}}

\textit{LearnSPGM} is apt at representing data
belonging to a single cluster, since (similarly to Chow-Liu trees) the edge weights are computed from a single mutual information matrix estimated on the whole dataset. To model densities with natural clusters one can use mixtures of SPGMs trained with  EM. We write a mixture of SPGMs in the form $\sum_{k=1}^{K}\lambda_{k}P_{k}(X|\theta^{k})$, where each term $P_{k}(X|\theta^{k})$ is the probability distribution encoded by a SPGM $\spgm_k$ (Eq. \ref{mixture-ptau}) governed by parameters $\theta^{k} = \{\mathcal{G}^k, \{P^k_{st}\}, \{W^k_{s}\}, \{Q^k_{s}\}\}$ (Definition \ref{def:ESPN}).

EM can be adopted for any mixture model as long as the weighted
maximum log-likelihood in the M-step can be solved (alg. \ref{MT_EM} line $8$). In addition, \citet{NealHinton98_incremEM}
show that EM converges as long as the M-step can be at least partially
performed, namely if it possible to find parameters $\overline{\theta}_k$ such that (see Eq. \ref{eq:wLL})
\begin{equation}
\text{WLL}({P_{k}}(X|\overline{\theta}_k),w^k) \ge \text{WLL}({P_{k}}(X|\theta^{k}),w^k).\label{eq:conv}
\end{equation}

These observation suggest to use \textit{LearnSPGM} to approximately solve the weighted maximum likelihood problem. 
However, while \textit{LearnSPGM} ensures that the Chow-Liu tree lower bound always increases, the actual weighted log-likelihood can decrease. To satisfy Eq. \eqref{eq:conv}, we employ the simple shortcut of rejecting
updates of component $\spgm_{k}$ when Eq. \eqref{eq:conv} is not
satisfied for $\theta_{new}^{k}$ (Algorithm \ref{alg:SPGM_MIX_EM} line $9$). Doing this, the following holds: 

\begin{proposition} The log-likelihood of the training set does not decrease at each iteration of EM for Mixtures of SPGMs (Alg. \ref{alg:SPGM_MIX_EM}). 
\end{proposition}



\begin{algorithm}[t]
\caption{\textit{LearnSPGM} $\left(D=\{x^i\},\{w^i\}\right)$}
\label{alg:LearnSPGM}
\begin{algorithmic}[1]
\Input samples $D$, optional sample weights $w$ 
\Output SPGM $\spgm$ approx. maximizing $\sum_{i=1}^{N}w_{i}\ln P(x^{i}|\theta)$
\State{$\mathbb{I}\leftarrow$ Mutual Information of $D$ with weights $w$} 
\State{${T^{*}}\leftarrow$ Chow-Liu tree with connection matrix $\mathbb{I}$} 
\State{ $\spgm\leftarrow$SPGM representing $T^{*}$} 
\State{$Q\leftarrow$ queue of edges $(s,t)\notin{\mathcal{E}}$ ordered by decreasing $\mathbb{I}_{st}$} 
\Repeat 
	\State{\textit{InsertEdge} $\left(\spgm,T^{*},Q.pop()\right)$} 
	\State{$AssignWeights$ to the modified Sum Node} 
\Until{convergence \Or max size reached}
\end{algorithmic} 
\end{algorithm}
\begin{algorithm}[tb]
\caption{EM for Mixtures of SPGMs(${\spgm_{k}},{\lambda_{k}},D$)} 
\label{alg:SPGM_MIX_EM}
\begin{algorithmic}[1] 
\Input Initial model $P(X)=\sum_k{\lambda_{k} \spgm_{k}(X)}$, samples $D$
\Output Updated ${\spgm_{k}(X)}$, $\lambda_{k}$ locally maximizing $\sum_{i=1}^{N}w_{i}\ln P(x^{i}|\theta)$
\Repeat 
	\ForAll{ $k\in{1...K},i\in{1...N}$} // E-step 
		\State $\gamma_{k}(i)\leftarrow\frac{\lambda_{k}\spgm_{k}(x_{i})}{\sum_{k'}\lambda_{k'}\spgm_{k'}(x_{i})}$ 
		\State $\Gamma_{k} \leftarrow \sum_{i=1}^{N} {\gamma_{k}(i)}$
		\State $w_{i}^{k}\leftarrow \gamma_{k}(i) / \Gamma_{k}$ 
	\EndFor 

	\ForAll {$k\in{1...K}$ } // M-step 
		\State $\lambda_{k}\leftarrow\Gamma_{k}/N$ 
		\State $\overline{\spgm}_k \leftarrow$ \textit{LearnSPGM}($D,w^k$)
		\If {$\text{WLL}(\overline{\spgm}_k,w^k) \ge \text{WLL}({\spgm_k},w^k)$}
			\State $\spgm_k \leftarrow \overline{\spgm}_k$ 
		\EndIf
	\EndFor

\Until{convergence} 

\end{algorithmic}
\end{algorithm}

\section{Empirical Evaluation\label{sec:Empirical-Evaluation}}

\begin{table*}[tb]
\caption{Dataset structure and test set log likelihood comparison.  \label{table:LLcomparison}}
\vspace{5pt}
 \centering{}\resizebox{0.9 \textwidth}{!}{%
\begin{tabular}{|l|ll|lllllllll|}
\hline 
Dataset  & \#vars  & \#train  & SPGM  & ECNet  & MergeSPN  & MCNet  & ID-SPN  & ACMN  & SPN  & MT  & LTM\tabularnewline
\hline 
NLTCS  & 16  & 16181  & \textbf{-5.99}  & -6.00  & -6.00  & -6.00  & -6.02  & -6.00  & -6.11  & -6.01  & -6.49\tabularnewline
MSNBC  & 17  & 291326  & \textbf{-6.03}  & -6.05  & -6.10  & -6.04  & -6.04  & -6.04  & -6.11  & -6.07  & -6.52\tabularnewline
KDDCup2K  & 64  & 180092  & -2.13  & -2.13  & \textbf{-2.12}  & \textbf{ -2.12}  & -2.13  & -2.17  & -2.18  & -2.13  & -2.18\tabularnewline
Plants  & 69  & 17412  & -12.71  & -12.19  & \textbf{-12.03}  & -12.74  & -12.54  & -12.80  & -12.98  & -12.95  & -16.39\tabularnewline
Audio  & 100  & 15000  & -39.90  & -39.67  & \textbf{-39.49}  & -39.73  & -39.79  & -40.32  & -40.50  & -40.08  & -41.90\tabularnewline
Jester  & 100  & 9000  & -52.83  & \textbf{-52.44}  & -52.47  & -52.57  & -52.86  & -53.31  & -53.48  & -53.08  & -55.17\tabularnewline
Netflix  & 100  & 15000  & -56.42  & -56.13  & \textbf{-55.84}  & -56.32  & -56.36  & -57.22  & -57.33  & -56.74  & -58.53\tabularnewline
Accidents  & 111  & 12758  & \textbf{-26.89}  & -29.25  & -29.32  & -29.96  & -26.98  & -27.11  & -30.04  & -29.63  & -33.05\tabularnewline
Retail  & 135  & 22041  & -10.83  & \textbf{-10.78}  & -10.82  & -10.82  & -10.85  & -10.88  & -11.04  & -10.83  & -10.92\tabularnewline
Pumsb-star  & 163  & 12262  & \textbf{-22.15}  & -23.34  & -23.67  & -24.18  & -22.40  & -23.55  & -24.78  & -23.71  & -31.32\tabularnewline
DNA  & 180  & 1600  & \textbf{-79.88}  & -80.66  & -80.89  & -85.82  & -81.21  & -80.03  & -82.52  & -85.14  & -87.60\tabularnewline
Kosarek  & 190  & 33375  & -10.57  & \textbf{-10.54}  & -10.55  & -10.58  & -10.60  & -10.84  & -10.99  & -10.62  & -10.87\tabularnewline
MSWeb  & 294  & 29441  & -9.81  & \textbf{-9.70}  & -9.76  & -9.79  & -9.73  & -9.77  & -10.25  & -9.85  & -10.21\tabularnewline
Book  & 500  & 8700  & -34.18  & \textbf{-33.78}  & -34.25  & -33.96  & -34.14  & -36.56  & -35.89  & -34.63  & -34.22\tabularnewline
EachMovie  & 500  & 4524  & -54.08  & -51.14  & \textbf{-50.72}  & -51.39  & -51.51  & -55.80  & -52.49  & -54.60  & \dag{}\tabularnewline
WebKB  & 839  & 2803  & -154.55  & -150.10  & \textbf{-150.04}  & -153.22  & -151.84  & -159.13  & -158.20  & -156.86  & -156.84\tabularnewline
Reuters-52  & 889  & 6532  & -85.24  & -82.19  & \textbf{-80.66}  & -86.11  & -83.35  & -90.23  & -85.07  & -85.90  & -91.23\tabularnewline
20Newsgrp.  & 910  & 11293  & -153.69  & -151.75  & \textbf{-150.80}  & -151.29  & -151.47  & -161.13  & -155.93  & -154.24  & -156.77\tabularnewline
BBC  & 1058  & 1670  & -255.22  & -236.82  & \textbf{-233.26}  & -250.58  & -248.93  & -257.10  & -250.69  & -261.84  & -255.76\tabularnewline
Ad  & 1556  & 2461  & \textbf{-14.30}  & -14.36  & -14.34  & -16.68  & -19.00  & -16.53  & -19.73  & -16.02  & \dag{} \tabularnewline
\hline 
\end{tabular}} 
\end{table*}

We evaluate SPGMs on $20$ real world datasets for density estimation
described in \citet{ArithmCircuitsLearning}. The number of variables
ranges from $16$ to $1556$ and the number of training examples from
$1.6K$ to $291K$ (Table \ref{table:LLcomparison}). All variables
are binary. We compare against several well cited state of the art
methods, referred to with the following abbreviations: MCNets (Mixtures
of CutsetNets, \citet{Rahman14Cnet}); ECNet (Ensembles of CutsetNets,
\citet{Rahman16}); MergeSPN \citet{Rahman16mergeSPN}; ID-SPN \citet{Rooshenas14};
SPN \citet{SPNstructureLearning2013}; ACMN \citet{ArithmCircuitsLearning},
MT (Mixtures of Trees \citet{LearningWithMixturesOfTrees}); LTM (Latent
Tree Models, \citet{Choi2011}).

\textbf{Methodology. } We found empirically that the best results
were obtained using a two phase procedure: first we run EM updates
with LearnSPGM on both structure and parameters until validation log-likelihood
convergence, then we fine-tune using EM for SPNs on parameters only
until convergence. We fix the following hyperparameters by grid search
on validation set: maximum number of edge insertions $\{10,20,60,120,400,1000,5000\}$,
mixture size $\{5,8,10,20,100,200,400\}$, uniform prior $\{10^{-1},10^{-2},10^{-3},10^{-9}\}$
(used for mutual information and sum weights). LearnSPGM was implemented
in C++ and is available at the following URL: \url{https://github.com/ocarinamat/SumProductGraphMod}.
The average learning time per dataset is $42$ minutes on an Intel
Core i5-4570 CPU with $16$ GB RAM. Inference takes up to one minute
on the largest dataset.

\textbf{Results. }Test set log-likelihood results averaged over $5$ random parameters initializations are shown in Table
\ref{table:LLcomparison}. Our methods performs best between all compared
models in $6$ \emph{datasets over} $20$ (for comparison, ECNets
are best in $5$ cases, MergeSPN in $9$, MCnets in $1$). Interestingly,
LearnSPGM compares well against a well established literature despite
being a radically \emph{novel }approach, since ECNet, MergeSPN, MCNet,
ID-SPN, ACMN, SPN all create a search tree by finding independences
in data recursively (see \citet{SPNstructureLearning2013}). LearnSPGM
is simpler than methods with similar performances: for instance ECnets
use boosting and bagging procedures for ensemble learning, evolving
over CNets that use EM, and MergeSPN post-processes the SPN in \citet{SPNstructureLearning2013}.
Our model can be improved by including these techniques, such as using
ensemble learning as in ECNets rather than EM as in MCnets. Finally,
notice that SPGMs - that are large mixtures of trees - \emph{always}
outperform standard mixture of trees. This confirms that sharing tree
structure \emph{helps preventing overfitting}, which is critical in
these models.

\section{Conclusions and Future Work}

We introduced Sum-Product Graphical Model (SPGM), a new architecture
bridging Sum-Product Networks (SPNs) and Graphical Models (GMs) by
inheriting the expressivity of SPNs and the high level probabilistic
semantics of GMs. The \emph{new connections} between the two fields
were exploited in a structure learning algorithm extending the Chow-Liu
tree approach. This algorithm is competitive with the \emph{state of the art} methods in density estimation despite using a novel approach and being the first algorithm to directly obtain DAG structured SPNs. 

The major relevance of SPGMs consists in providing an interpretation of SPNs that has the semantics of graphical models, and thus allows to connect the two worlds without compromising their efficiency or compactness. These connections have been exploited preliminarily in our structure learning algorithm, but -- more interestingly -- they open up several direction of future research that should be explored in future work. 

The most interesting of these directions seems to us to be the approximation of intractable GMs by using very large but tractable mixture models implemented by SPGMs. This approach is mainly motivated by the success of  mixtures of chain graphs  for approximate inference in the Tree-Reweighted Belief Propagation (TRW) framework (see \citet{Kolmogorov2006}). In this field, mixtures of chain models with shared components are used to efficiently evaluate updates for the parameters of an intractable graphical model. SPGMs seem to be a very promising architecture to apply to this framework due to their ability to \emph{efficiently} represent \emph{very large} mixtures of trees due to shared parts, which might allow to replace the relatively small mixture of chains used in TRW with large mixtures of more complex models (trees). This aspect should be subject of follow-up work.

\appendix

\section{Appendix}

\subsection{Proof of Proposition \ref{prop:subIsTree}.\label{subsec:Proof-of-Proposition-1}}

Consider a subtree $\tau\in\enctr(\spgm)$ as in Definition
\ref{def:ESPNsubn}. 
\begin{enumerate}
\item Let us first consider messages generated by sum nodes $s$. Considering
that $s$ has only one child $ch(s)$ in $\tau$ (for Definition \ref{def:ESPNsubn}),
corresponding to indicator variable $[Z_{s}]_{k_{s}}$ , and applying
Eqs. \ref{eq:messages}, the messages for observed and unobserved
sum nodes are as follows:
\begin{eqnarray}
\mu_{st;j}= & [Z_{s}]_{k_{s}}Q_{s}(k_{s})\mu_{ch(s),t;j}, & \text{Observed Sum Node},\label{eq:msgSumNode-1}\\
\mu_{st;j}= & W_{s}(k_{s})\mu_{ch(s)_{k},t;j}, & \text{Unobserved Sum Node}.\label{eq:msgSumUnobs-1}
\end{eqnarray}
Hence sum messages contribute only
by introducing a multiplicative term $[Z_{s}]_{k_{s}}Q_{s}(k_{s})$
or $W_{s}(k_{s})$. Now, all variables in the set $Z$ appear in $\tau$ (Proposition \ref{prop:scopeSubtree}). 
Hence, the sum nodes together contribute with the following multiplicative
term:
\begin{equation}
\prod_{s\in O(\tau)}Q_{s}(k_{s,\tau})\prod_{s\in U(\tau)}W_{s}(k_{s,\tau})\prod_{[Z_{s}]_{k_{s}}\in z_{\tau}}[Z_{s}]_{k_{s}}=\lambda_{\tau}\prod_{[Z_{s}]_{k_{s}}\in z_{\tau}}[Z_{s}]_{k_{s}}.\label{eq:coeffjudfhisah}
\end{equation}
From this it follows that message passing in $\tau$ is equivalent to message passing in a SPGM $\overline{\tau}$ obtained by \emph{discarding} all the sum nodes from $\tau$, 
followed by multiplying the resulting messages for Eq. \ref{eq:coeffjudfhisah}. 

\item Consecutive Product Nodes in $\overline{\tau}$ can be merged by adding
the respective children to the parent Product Node. In addition, between
each sequence Vnode-Vnode in $\overline{\tau}$ we can insert a product
node with a single child. Thus, we can take $\overline{\tau}$ as
containing only Vnodes and Product Nodes, such that the children of
Product Nodes are Vnodes. 
Putting together Eqs. \ref{eq:msgProd} and \ref{eq:msgPnode},
the message passed by each Vnode $t$ to $s\in \vpa(t)$ is: 
\begin{equation}
\mu_{t \rightarrow s;j} = \sum_{k\in\val\left(X_{t}\right)}P_{s,t}\left(k|j\right)[X_{t}]_{k} \prod_{q\in ch\left(ch(c)\right)}\mu_{q \rightarrow t;j}.\label{eq:prodVnodeMsg}
\end{equation}
Notice that the input messages are generated from the grandchildren of $t$, that is the children of the Product Node child of $t$. 
This corresponds to the message passed by variables in a tree graphical model obtained by removing the Product Nodes from $\overline{\tau}$ and attaching the children of Product Nodes (here, elements $q\in ch\left(ch(c)\right)$) as children of their parent Vnode (here, $t$), which can be seen by noticing the equivalence to Eq. \eqref{eq:treeMsgPass}. This tree GM can be immediately identified as $P_{\tau}(X)$.
Note also that if the root of $\overline{\tau}$ is a Product Node, then $P_{\tau}(X)$ represents a forest of trees, one for each child of the root Product Node. 
 
\item The proof is concluded reintroducing the multiplicative factor in
Eq. \ref{eq:coeffjudfhisah} discarded when passing from $\tau$ to
$\overline{\tau}$.

\end{enumerate}
\qed


\subsection{Proof of Proposition \ref{prop:exp}.\label{Proof-of-prop:exp}}

The proposition can be proven by inspection, considering a SPGM $\spgm$ built by stacking units as follows:
Sum Node $s_1$ (the root of $\spgm$) is associated to observed variable $Z_1$ and has as children the set of Vnodes $V_1 = \{v_{1,1},v_{1,2},\dots,v_{1,M} \}$. All the nodes in $V_1$ have a common single child, which is the sum node $s2$. In turn, $s_2$ has the same structure of $s_1$, having Vnode children $V_1 =  \{ v_{2,1},v_{2,2},\dots,v_{2,M}\}$ which are connected to a single child $s3$, and so forth for Sum Nodes $s_3,s_4,\dots, s_K$. The number of edges in $\spgm$ is $2MK$. On the other hand, a different subtree can be obtained for each choice of active child at each sum node. There is a combinatorial number of such choices, thus the number of different subtrees is $K^M$. \qed 

\subsection{Proof of Proposition \ref{prop:infInEncTtr}.\label{subsec:Proof-of-Proposition}}


First, consider a SPGM $\spgm(X,Z)$ defined over $\mathcal{G}=\mathcal{V},\mathcal{E}$. Let us
take a node $t \in \mathcal{V}$ and 
let $\spgm_t(X^t, Z^t)$ denote the sub-SPGM rooted in $t$. Suppose that $\spgm_t$ satisfies the proposition and hence it can be written as:
\begin{equation}\label{eq:form1}
\spgm_t = \sum_{\tau_t \in \enctr(\spgm_t)} \tau_t.
\end{equation}
Let us consider messages sent from node $t$ to a Vparent $s \in \vpa(t)$ (Definition \ref{def:Vparent}), and note that if form \eqref{eq:form1} is satisfied then messages take the form
\begin{equation}\label{eq:muMix}
\mu_{t\rightarrow s;j} = \sum_{\tau_t \in \enctr(\spgm_t)} \mu_{root(\tau_t) \rightarrow s;j}, 
\end{equation}
where $root(\tau_t)$) denotes the message sent from the root of subtree $\tau_t$ to $s$. 
Vice versa, if form \eqref{eq:muMix} is satisfied then $\spgm_t$ can be written as Eq. \eqref{eq:form1}. This extends also to the case $\vpa(t)=\emptyset$, in which messages are sent to a fictitious $root$ node (Definition \ref{def:Message-passing-in}). 

The proposition can now be proved by induction. First, the base case: sub-SPGMs rooted at Vnodes leaves trivially assume form \eqref{eq:form1}, and hence also the form \eqref{eq:muMix}. Then, the inductive step: Lemma \eqref{lemma2} can be applied recursively for all nodes from the leaves to the root. Hence, $\spgm$ assumes the form of Eq. \eqref{eq:muMix} and thus also of \eqref{eq:form1}. \qed

\begin{lemma}\label{lemma2}
Consider a node $t \in \mathcal{V}$. 
Suppose that the messages sent from the children of node $t\in \mathcal{V}$ are in the form \eqref{eq:muMix}. 
Then,  each message sent from $s$ also assumes the form \ref{eq:muMix}. 
\end{lemma}

Let us suppose, for simplicity, that $t$ has only two children $p$ and $q$ - the extension to the
general case is straightforward. We distinguish the three cases of $t$ being a Sum Node, a Product Node or a Vnode. 
\begin{itemize}
\item Suppose $t$ is an Observed Sum Node with weights $\left[Q_{t}(0),Q_{t}(1)\right]$ and indicator variables $[Z_t]_0$ and $[Z_t]_1$ associated to each child. By Eq. \eqref{eq:msgSumUnobs-1}, the children send messages to their Vparent $s$, and since input message are in the form \eqref{eq:muMix}, $t$ sends the following message:
\[
\mu_{t \rightarrow s;j}=  
Q_{t}(0)[Z_t]_0 \sum_{\tau_p \in \enctr(\spgm_p)} \mu_{root(\tau_p) \rightarrow s;j}
+
Q_{t}(1)[Z_t]_1 \sum_{\tau_q \in \enctr(\spgm_q)} \mu_{root(\tau_q) \rightarrow s;j}
.\]
This is again in the form \ref{eq:muMix}. This can be seen because each term in the sum is in the form $Q_{t}(0)[Z_t]_0\enctr(\spgm_p)$ which corresponds to message passed from the subtree $\tau_t \in \enctr(\spgm_t)$ obtained choosing child $p$ (in this case). It is easy to see that terms corresponding to each subtrees of $t$ are present. 

\item If $t$ is an Unobserved Sum Node the discussion is identical to the point above. 

\item If $t$ is a Product Node, then the children send messages to their Vparent $s$, and since input message are in the form \eqref{eq:muMix}, $t$ sends the following message:
\[
\mu_{t \rightarrow s;j}= 
\bigg( \sum_{\tau_p \in \enctr(\spgm_p)} \mu_{root(\tau_p) \rightarrow s;j} \bigg) \bigg( \sum_{\tau_q \in \enctr(\spgm_q)} \mu_{root(\tau_q) \rightarrow s;j} \bigg)
=
\sum_{ (\tau_{p},\tau_{q})\in\enctr(\spgm_p) \times \enctr(\spgm_q)} \mu_{root(\tau_p) \rightarrow s;j}  \mu_{root(\tau_q) \rightarrow s;j}
.\]
Now, $ \mu_{root(\tau_p) \rightarrow s;j}  \mu_{root(\tau_q) \rightarrow s;j}$ can be seen as the message generated by a particular subtrees of $t$, and thus node $\mu_{t \rightarrow s;j}$ is in the form \eqref{eq:muMix}. 

\item If $t$ is a Vnode, then it has a \emph{single child} $p$ by definition
\ref{def:ESPN}, sending messages to $t$. Thus, the input message is in the form \eqref{eq:muMix}, and $t$ sends the following message: 
\begin{align*}
\mu_{t \rightarrow s;j}= & \sum_{k\in\val\left(X_{t}\right)} P_{s,t}\left(k|j\right)[X_{t}]_{k} \sum_{\tau_p \in \enctr(\spgm_p)} \mu_{root(\tau_p) \rightarrow t;k} \\ 
 & = \sum_{\tau_p \in \enctr(\spgm_p)} \sum_{k\in\val\left(X_{t}\right)} P_{s,t}\left(k|j\right)[X_{t}]_{k} \mu_{root(\tau_p) \rightarrow t;k}.
\end{align*}
Let us analyze a term of the sum for each fixed $\tau_p$. Such term corresponds to the root message of a SPGM $\overline{\spgm}$ obtained taking $\tau_p$ and adding the Vnode $t$ as parent of $p$ (due to Eq. \eqref{eq:msgPnode}).  But $\overline{\spgm}$ obtained in this form is a subtree of $t$ by Definition \ref{def:ESPNsubn}. Therefore, taking the sum $\sum_{\tau_p \in \enctr(\spgm_p)}$ corresponds to summing over messages sent by all subtrees of $t$, in the form  \eqref{eq:muMix}.  \qed

\end{itemize}

\subsection{Proof of Proposition \ref{prop:spgmAsSpn}\label{proof:asSpn}}

The proposition can be proven by induction showing that if input messages represent valid SPNs, then also the output SPGM messages are SPNs (see Fig. \ref{fig:Representation-of-messages}).  
Formally, it is sufficient to notice that the hypothesis of Lemma \ref{lemma1} below is trivially true for leaf messages (inductive hypothesis), hence the lemma can be inductively applied for all nodes up to the root. \qed

\begin{lemma}\label{lemma1}
Consider nodes $t\in\mathcal{V}$ and $s\in \vpa(t)$, and let $X^{t}$ and $Z^{t}$ respectively denote the $X$ scope and $Z$ scope of node $t \in \mathcal{V}$. 
 If the message $\mu_{t\rightarrow s;j}$ encodes
a SPN $S_{t;j}(X^{t},Z^{t})$, then the message $\mu_{s\rightarrow r;i}$ sent from node $s$ to any node $r \in \vpa(s)$ also implements
a SPN $S_{r;i}(X^{s},Z^{s})$. This also holds when $\vpa(s)=\emptyset$, where  $r$ is simply replaced by the fictitious $root$ node (Definition \ref{def:Message-passing-in}. 
\end{lemma}
\emph{Proof.} Consider the message passing Eqs. \ref{eq:messages}, referring to
Fig. \ref{fig:Representation-of-messages} for visualization. 
We distinguish the case of node $t$ being
a Sum, Product and Vnode. 
\begin{itemize}
\item Sum Nodes (Observed and non-Observed). First, we note that every child of a Sum Node has
the same scope $X^{s},Z^{s}$ (for Definition \ref{def:ESPN} condition
\ref{enu:sgm3}). Employing the hypothesis, the generated
message is:
\begin{eqnarray}
\mu_{t \rightarrow s;j}= & \sum_{k=1}^{|ch(t)|}[Z_{t}]_{k}Q_{t}(k)S_{t;j}\left(X^{t},Z^{t}\right), & \text{$t$ is a Sum Node, Observed}, \label{eq:24356432df}\\
\mu_{t \rightarrow s;j}= & \sum_{k=1}^{|ch(t)|}W_{t}(k)S_{t;j}\left(X^{t},Z^{t}\right), & \text{$t$ is a Sum Node, not-Observed}. \label{eq:234sd}
\end{eqnarray}
It is straightforward to see that both equations represent valid SPNs:
the sum $\sum_{k=1}^{|ch(t)|}$ becomes the root SPN Sum Node with
non-negative weights $Q_{t}(k)$ and $W_{t}(k)$ respectively, and
its children are SPNs having the same scope (in the form $[Z_{t}]_{k}\otimes S_j\left(X^{t},Z^{t}\right)$
for Eq. \ref{eq:24356432df} and in the form $S_j\left(X^{t},Z^{t}\right)$
for Eq. \ref{eq:234sd}). Note that $Z^{t}\cap Z_{t}=\emptyset$ for
Definition \ref{def:ESPN} condition \ref{enu:sgm3}, therefore condition
\ref{enu:The-weighted-sum} in Definition \ref{def:SPN} is satisfied. 

\item Product Nodes. Applying the hypothesis to input messages, Eq. \ref{eq:msgProd} becomes:
\[\mu_{t \rightarrow s;j}= \prod_{q\in ch\left(t\right)} S_{q;j}\left(X^{q},Z^{q}\right).\]
This represents a valid SPN with a Product Node as root since the children node's scopes are disjoint
(for Definition \ref{def:ESPN} condition \ref{enu:sgm2}), and thus
condition \ref{enu:The-product-} in Definition \ref{def:SPN} is
satisfied. 

\item Vnodes. Applying the hypothesis to input messages, Eq. \ref{eq:msgPnode} becomes: 
\[ \mu_{t \rightarrow s;j} =  \sum_{k\in\val\left(X_{t}\right)}P_{s,t}\left(k|j\right)[X_{t}]_{k}  S_{ch(t);k} \left(X^{ch(t)},Z^{ch(t)} \right). \]
This represents a valid SPN with a Sum Node $\sum_{k\in\val\left(X_{t}\right)}$
as root. To see this, first note that terms $P_{s,t}\left(k|j\right)$ can
be interpreted as weights. The Sum Node is connected to children SPNs in the form $[X_{t}]_{k} \otimes  S_{ch(t);k}\left(X^{ch(t)},Z^{ch(t)} \right)$, which are valid SPNs since $X^{i}\cap X_{s}=\emptyset$ and thus
condition \ref{enu:The-product-} in Definition \ref{def:SPN} is satisfied. 
In addition all child SPNs have the same scope for Definition \ref{def:ESPN} condition \ref{enu:sgm3}, hence condition \ref{enu:The-weighted-sum} in Definition \ref{def:SPN} is satisfied for the Sum Node. 
\end{itemize} \qed

\bibliographystyle{apalike}
\bibliography{bibliography}

\end{document}